\documentclass[lettersize,journal]{IEEEtran}
\usepackage{amsmath,amsfonts}
\usepackage{algorithmic}
\usepackage{algorithm}
\usepackage{array}
\usepackage[caption=false,font=normalsize,labelfont=sf,textfont=sf]{subfig}
\usepackage{textcomp}
\usepackage{stfloats}
\usepackage{url}
\usepackage{verbatim}
\usepackage{graphicx}
\usepackage{cite}
\usepackage[utf8]{inputenc} % allow utf-8 input
\usepackage[T1]{fontenc}    % use 8-bit T1 fonts
\usepackage{url}            % simple URL typesetting
\usepackage{booktabs}       % professional-quality tables
\usepackage{amsfonts}       % blackboard math symbols
\usepackage{nicefrac}       % compact symbols for 1/2, etc.
\usepackage{microtype}      % microtypography
\usepackage{xcolor}         % colors

\usepackage{algorithm}
\usepackage{algorithmic}
\usepackage{csquotes}
\usepackage{amsmath}
\usepackage{enumitem}
\usepackage{booktabs}
\usepackage{tabularx}
\usepackage{adjustbox}
\usepackage{multirow}
\usepackage{pgfplots}
\usepgfplotslibrary{groupplots}
\pgfplotsset{compat=1.18}
\usepackage{hyperref}
\usepackage{xcolor}
\usepackage{listings}

\usepackage{subcaption} 

\newcommand{\kc}[1]{{\color{black}#1}}

\usepackage{amssymb}
\usepackage{mathtools}
\usepackage{amsthm}
\theoremstyle{plain}
\newtheorem{theorem}{Theorem}[section]

\theoremstyle{definition}

\theoremstyle{remark}

\begin{document}

\title{GVPO++: Group Variance Policy Optimization for LLM Post-Training and On-Policy Distillation}

\author{Kaichen Zhang, Yuzhong Hong, Junwei Bao, Hongfei Jiang, Yang Song, Dingqian Hong, and \\ Hui Xiong,~\IEEEmembership{Fellow,~IEEE}
        % <-this % stops a space
}

\markboth{Journal of \LaTeX\ Class Files,~Vol.~14, No.~8, August~2021}%
{Shell \MakeLowercase{\textit{et al.}}: A Sample Article Using IEEEtran.cls for IEEE Journals}


\maketitle

\vspace{-5mm}
\begin{abstract}
Post-training plays a pivotal role in enhancing the reasoning capabilities and task-specific expertise of large language models (LLMs). Despite recent advances in post-training methods, such as Group Relative Policy Optimization (GRPO), their practical deployment remains impeded by training instability arising from the reliance on \textit{importance sampling}.

We introduce Group Variance Policy Optimization (GVPO), a novel post-training method that integrates the analytical solution of KL-constrained reward maximization into its gradient weighting scheme. This formulation provides an intuitive interpretation: GVPO's gradient corresponds to the mean squared error between the central distance of implicit rewards and that of actual rewards. GVPO offers two key advantages: (1) it guarantees a unique optimal solution, exactly to the KL-constrained reward maximization objective, and (2) it enables flexible sampling distributions without requiring importance sampling.

Beyond general post-training, we show that GVPO naturally extends to on-policy distillation (OPD). 
Furthermore, GVPO enables the optimization of a broad family of extended OPD objectives, providing a principled foundation for diverse objective design.
By unifying theoretical guarantees with practical adaptability, GVPO establishes a new paradigm for reliable and versatile LLM post-training and on-policy distillation.

\end{abstract}

\section{Introduction}
Large language models (LLMs) \cite{zhao2025surveylargelanguagemodels,zheng2026lifelong}, trained on extensive datasets, exhibit impressive general-purpose capabilities, yet their practical utility depend critically on post-training \cite{tie2025surveyposttraininglargelanguage} refinement. While pre-training \cite{zhou2023comprehensivesurveypretrainedfoundation, qian2025beyond} equips LLMs with broad linguistic patterns, post-training techniques—such as supervised fine-tuning (SFT) \cite{ouyang2022training} and reinforcement learning \cite{miao2024optimizing, bai2022training} are indispensable for adapting these models to specialized applications and enhancing their reasoning capabilities \cite{kumar2026llm}.

Recent advances in post-training have been exemplified by Group Relative Policy Optimization (GRPO) \cite{shao2024deepseekmath}. Unlike conventional reinforcement learning frameworks \cite{schulman2017proximal}, which require training a separate value function, GRPO estimates the advantage by normalizing reward scores across a group of sampled responses. By eliminating the need for an auxiliary value model—whose computational and memory requirements can be comparable to those of the policy model—GRPO substantially reduces training overhead. This enables more efficient sampling and improves the scalability of reinforcement learning for large language models. Notably, DeepSeek-R1 \cite{guo2025deepseek} adopts GRPO in its post-training pipeline and demonstrates its effectiveness in achieving strong reasoning performance.

\begin{figure*}[t]
\centering
\includegraphics[width=\textwidth]{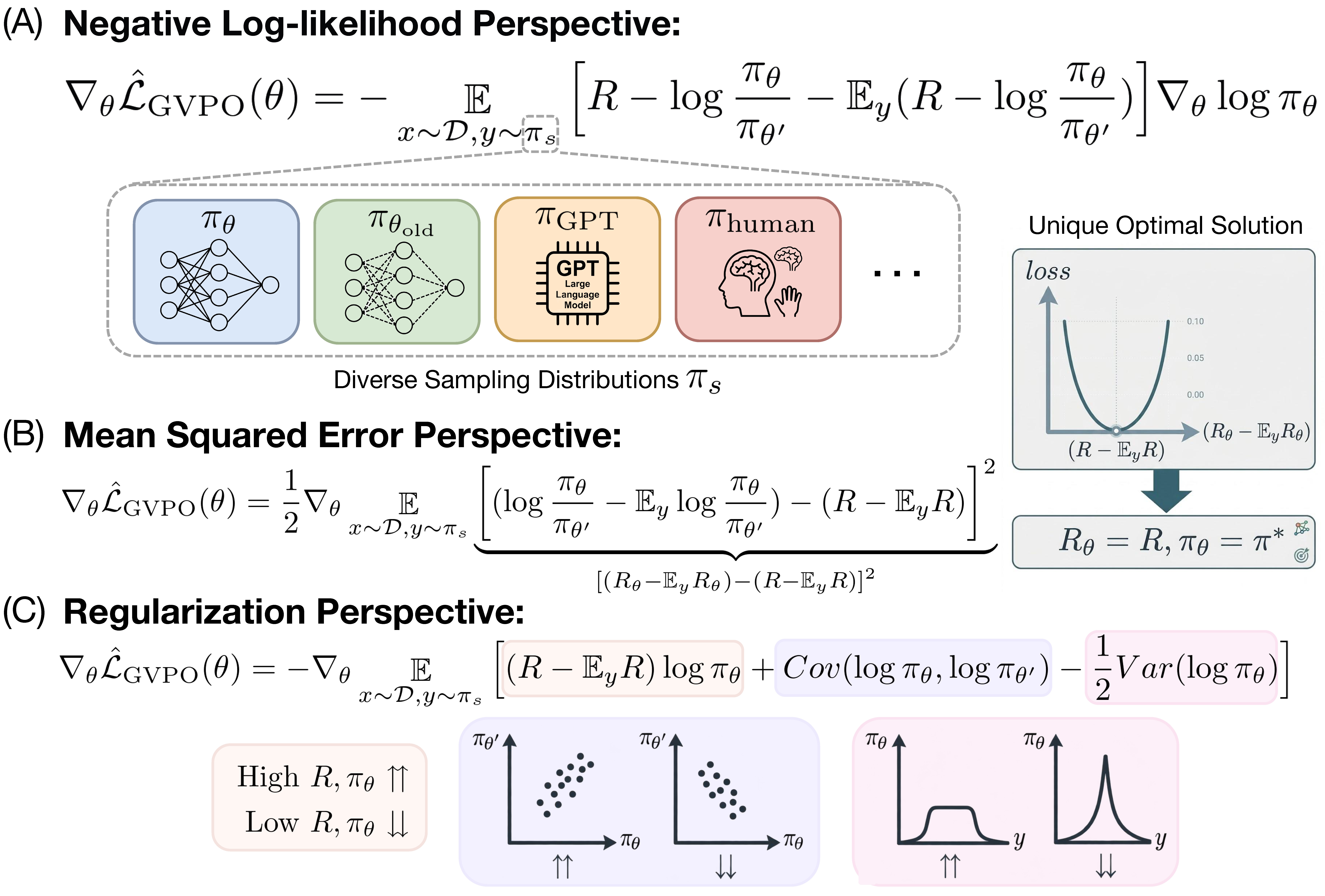}
\vspace*{-5mm}
\caption{\textbf{Three equivalent formulations of the GVPO loss provide distinct interpretations}: (A) the negative log-likelihood perspective highlights that GVPO can accommodate diverse sampling distributions without requiring importance sampling; (B) the mean squared error perspective reveals GVPO’s unique optimum, which simultaneously maximizes the reward subject to a KL-divergence constraint; and (C) the regularization perspective shows GVPO’s implicit regularization terms, which promote stable policy optimization. We set the hyperparameter $\beta=1$ for simplicity.}
\label{fig:three}
\vspace{-5mm}
\end{figure*}

However, prior work has documented training instability in GRPO \cite{yu2025dapo, liu2025understanding}. Specifically, GRPO relies on importance sampling to estimate gradients from off-policy samples, introducing an importance sampling ratio that can become arbitrarily large. Such unbounded ratios can lead to high-variance and unstable gradient estimates, undermining the robustness of GRPO and limiting its broader practical adoption. To address this issue, we propose \textbf{Group Variance Policy Optimization (GVPO)}, a novel approach designed to enable more reliable and versatile LLM post-training.

Our analysis begins with a key observation: post-training algorithms—including but not limited to SFT, Reject Sampling \cite{touvron2023llama}, and GRPO—share a unified mathematical structure in their loss gradients \cite{shao2024deepseekmath, gao2024towards}. Specifically, each method’s gradient can be expressed as a weighted sum of the gradients of the log-likelihoods of responses. 
This unified framework reveals that we can directly design weights to encode preferences—positive weights amplify gradients for favored responses, while negative weights suppress disfavored ones, with magnitudes modulating the strength of preference.

Motivated by the success of Direct Preference Optimization (DPO) \cite{rafailov2023direct}—which utilizes a closed-form link between reward models and the optimal policy under KL-divergence constraints \cite{jaques2017sequence}—we explore how to leverage this analytical relationship. A central obstacle arises from the partition function in the closed-form formula, which requires intractable expectation calculations over all possible responses. To address this, we identify a critical condition: 
\textit{when the sum of assigned response weights within a prompt group equals zero, the partition function becomes invariant across compared responses},
effectively canceling out in the policy update rule. This insight eliminates the need for explicit estimation of the partition function, thereby enabling deployment of the closed-form optimal policy while retaining its theoretical advantages.

Based on the previous findings, we design GVPO's weighting scheme where the gradient weight of a response in a group is the difference between the central distance of implicit rewards-which derive from the current policy and the reference policy-and that of actual rewards, illustrated in Figure~\ref{fig:three} (A).

We demonstrate that GVPO loss function carries physically meaningful interpretations. Specifically, we establish that its gradient equals that of a mean squared error loss measuring the discrepancy between implicit and actual reward central distances, illustrated in Figure~\ref{fig:three} (B).

In addition, GVPO can be interpreted through the lens of regularization, as illustrated in Figure~\ref{fig:three} (C). Specifically, its loss function can be decomposed into three components: (1) a group-relative advantage term that promotes advantage maximization by favoring responses with higher expected returns; (2) a covariance term between the current and reference policies that constrains excessive policy deviation and thereby stabilizes policy updates; and (3) a variance term of the current policy that encourages moderate entropy, naturally balancing exploration and exploitation.

We demonstrate that GVPO offers two key advantages:
\begin{itemize}[leftmargin=0.5cm]
\item GVPO has a unique optimal solution, which coincides precisely with the optimal solution of the KL-constrained reward maximization. This guarantee gives GVPO a significant theoretical advantage over DPO. Prior work \cite{bong2022generalized,hong2024energybasedpreferencemodeloffers} shows that DPO may fail to align with the optimal policy for the KL-constrained reward maximization problem due to limitations of the Bradley-Terry model \cite{wu2022asymptoticcomparisonidentifyingconstraints}. 
 \item GVPO supports flexible sampling distributions that avoids importance sampling and on-policy limitations. GVPO retains theoretical guarantees for the unique optimal solution under any sampling distribution satisfying a mild condition. Unlike on-policy approaches \cite{singh2000convergence,williams1992simple}, GVPO facilitates efficient off-policy training using diverse samplings. In contrast to off-policy methods \cite{shao2024deepseekmath,schulman2017proximal} reliant on importance sampling, GVPO inherently avoids gradient explosion risks without introducing bias through clipping techniques. 
 \end{itemize}

Beyond general post-training, we show that GVPO naturally extends to on-policy distillation (OPD) by reformulating its reverse KL divergence objective as the KL-regularized reinforcement learning objective studied in general post-training. In addition to retaining the aforementioned advantages of GVPO, its application to OPD operates at the sequence level, substantially broadening its applicability by enabling distillation between models from different model families or series, whose tokenizers and vocabularies are incompatible.

Furthermore, GVPO enables the optimization of a broad class of extended reverse KL objectives by allowing each response to be assigned a weighting scheme that induces a corresponding Boltzmann distribution over responses. Crucially, GVPO makes such optimization tractable by eliminating the need to estimate partition functions or perform infeasible sampling from these Boltzmann distributions. GVPO thus provides a principled foundation for designing diverse OPD objectives. In particular, weighting schemes can be constructed to suppress or promote specific undesirable or desirable response patterns during distillation, while retaining the theoretical guarantee that the teacher policy is its exact optimization destination.

In summary, by unifying theoretical guarantees with practical adaptability, GVPO establishes a new paradigm for reliable and versatile LLM post-training and on-policy distillation.
\section{Related Work}

\textbf{Post-training}. This paper closely relates to the LLM post-training \cite{tie2025surveyposttraininglargelanguage, gong2025pushing} literature. Post-training refers to the process of further optimizing a pretrained language model \cite{zhou2023comprehensivesurveypretrainedfoundation} to adapt the model to desired behaviors and downstream tasks \cite{besta2025demystifying}.

A prominent paradigm for post-training is Reinforcement Learning from Human Feedback (RLHF) \cite{bai2022training}, which typically consists of three stages: Supervised Fine-tuning (SFT) \cite{ouyang2022training, xu2026parameter}, reward modeling, and reinforcement learning. First, the pretrained model is fine-tuned on instruction–response pairs. Next, a reward model is trained from human preference comparisons to assess response quality. Finally, the policy model is optimized against the learned reward, commonly using Proximal Policy Optimization (PPO) \cite{schulman2017proximal}, to better align its behavior with human preferences. Reinforcement Learning with Verifiable Rewards (RLVR) tasks use verifiable outcomes, such as mathematical correctness, as reward signals. RLVR has become a common evaluation setting for post-training methods, as verifiable rewards avoid the reward hacking \cite{skalse2022defining} issues commonly encountered in RLHF, providing a reliable objective to assess their effectiveness.

Group Relative Policy Optimization (GRPO) \cite{shao2024deepseekmath, zhang2025rspo} simplifies PPO by removing the need for a separate value model. For each prompt, GRPO samples a group of responses and computes their relative advantages using the group-level rewards, which are then used to update the policy. This design reduces memory and computational costs while being particularly effective for LLM post-training. 
Compared with GRPO and its variants, such as Dr.GRPO \cite{liu2025understanding}, Remax \cite{li2023remax}, Reinforce++ \cite{hu2025reinforce++}, DAPO \cite{yu2025dapo}, our proposed GVPO does not rely on importance sampling and support diverse sampling distributions, while maintaining theoretical guarantees.

Compared with our preliminary conference paper \cite{zhang2026gvpo}, this journal version substantially extends GVPO from general reward-maximization post-training to on-policy distillation (OPD). In addition, we have strengthened the theoretical analysis, expanded the experimental evaluation, and improved the clarity and overall presentation of the manuscript.

\textbf{On-policy Distillation.} Model distillation \cite{gou2021knowledge} transfers knowledge from a teacher to a student. A common approach uses teacher-generated responses as targets for SFT or minimizing the forward KL divergence from the teacher to the student. Large language model OPD \cite{song2026surveyonpolicydistillationlarge} instead queries the teacher on responses sampled from the student’s own policy and typically optimizes a reverse KL objective. By training on the student’s on-policy distribution, OPD alleviates exposure bias by reducing the mismatch between the states encountered during training and those encountered at inference time. 

The representative OPD methods include MiniLLM \cite{gu2024minillm} and GKD \cite{agarwal2024policy}, which employ approximation-based or heuristic designs and require the teacher and student models to share the same vocabulary and tokenizer. Our paper contributes to the OPD literature by proposing a sentence-level algorithm that offers flexible usage while providing theoretical guarantees. Furthermore, our method enables the optimization of a broad class of extended reverse KL objectives, providing a principled foundation for the design of diverse OPD objectives.

\newpage
\section{Preliminary}
Large language models take a prompt $x$ as input and generate a response $y$ as output. A policy $\pi_{\theta}(y_t|x,y_{<t})$ with parameter $\theta$ maps a sequence of tokens generated ($x$ and $y_{<t}$) to a probability distribution over the next token $y_t$. We also denote $\pi_{\theta}(y|x)$ as the probability of generating the response $y$ from $x$. A reward model $R(x,y)$ scores the response $y$ as the reply to the prompt $x$, which can be a trainable function that implicitly reflects human preferences; or explicit evaluation ratings of human beings; or a verifiable outcome, such as correctness.

The general purpose of post-training of large language model is summarized as following: Given a dataset of prompts $x \sim D$, a reward model $R$, 
the objective is to train a policy $\pi_{\theta}$ that generates responses with higher rewards, that is, maximize $E_{x\sim D,y\sim \pi_{\theta}(\cdot|x)}R(x,y)$. 
In practice, to improve training stability and prevent the policy from deviating excessively from a reference policy $\pi_{\theta^\prime}$, post-training commonly employs the following KL-regularized objective:
\begin{equation}
\label{equation:rewardkl}
\max_{\pi_{\theta}} \mathop{\mathbb{E}}_{x\sim\mathcal{D},y\sim\pi_\theta(y|x)}[R(x,y)]-\beta\mathbb{D}_{KL}[\pi_\theta(y|x)||\pi_{\theta^\prime}(y|x)]
\end{equation}
where $\beta > 0$ controls the divergence penalty from policy $\pi_{\theta^\prime}$.

\textbf{Optimal solution to the KL-regularized objective.}
In RLHF, pairwise response preferences $(x,y_w,y_l)$ are used for reward model training, where $y_w$ denotes the preferred response and $y_l$ the dispreferred response to prompt $x$.

Rather than employing separate reward modeling and policy optimization stages, Direct Preference Optimization (DPO) \cite{rafailov2023direct} derives a single-stage training paradigm by exploiting the analytical relationship between optimal policies and reward functions. The optimal solution to Equation~\ref{equation:rewardkl} satisfies:
\begin{equation}
\label{equation:optimal}
\pi^*(y|x)=\frac{1}{Z(x)}\pi_{\theta^\prime}(y|x)e^{R(x,y)/\beta}
\end{equation}
which implies the corresponding reward function:
\begin{equation}
\label{equation:Rxy}
R(x,y)=\beta \log\frac{\pi^*(y|x)}{\pi_{\theta^\prime}(y|x)}+\beta \log Z(x)
\end{equation}
where $Z(x)=\sum_y\pi_{\theta^\prime}(y|x)e^{R(x,y)/\beta}$ represents the partition function. DPO circumvents explicit computation of $Z(x)$ by substituting the reward expression from Equation~\ref{equation:Rxy} into the Bradley-Terry loss \cite{bradley1952rank}, yielding the final objective:
\begin{equation} \nonumber
\label{equation:DPO}
\resizebox{\hsize}{!}{$
\begin{aligned}
\mathcal{L}(\theta) = -\sum_{(x, y_w, y_l) \in \mathcal{D}}\log \sigma(\beta \log\frac{\pi_\theta(y_w|x)}{\pi_{ref}(y_w|x)}- \beta \log\frac{\pi_\theta(y_l|x)}{\pi_{ref}(y_l|x)})
\end{aligned}
$}
\end{equation}

\textbf{Unified framework of post-training.} Post-training algorithms are observed to share a unified framework \cite{shao2024deepseekmath, gao2024towards}, in which their losses' gradients share a same format:
\begin{equation}
\nabla_\theta\mathcal{L}(\theta) = -\sum_{(x, y_1,y_2,..,y_k) \in \mathcal{D}} \sum_{i=1}^k w_i \nabla_\theta\log \pi_\theta(y_i|x) \label{equation:pt_framework}
\end{equation}
SFT typically has its $w=1$. GRPO's essential weights are the standard scores of its rewards in a prompt group, $w_i =\frac{R_i-\overline{R}}{\text{std}(R)}$. Though it is not obvious for DPO, its gradients also share the same format, in which $w_w = \sigma(\beta \log\frac{\pi_\theta(y_l|x)}{\pi_{ref}(y_l|x)}- \beta \log\frac{\pi_\theta(y_w|x)}{\pi_{ref}(y_w|x)})$ and $w_l=-w_w$. This unified framework for post-training follows from the chain rule of derivatives.

\section{Group Variance Policy Optimization}
\subsection{Motivation}

The unified post-training framework (Equation~\ref{equation:pt_framework}) indicates that response preferences can be directly incorporated through the assignment of weights $w_i$.

To determine appropriate weights, we draw inspiration from the success of DPO. In particular, we aim to exploit the closed-form relationship between rewards and the optimal solution to the KL-constrained reward maximization objective: 
\begin{equation}
R_\theta(x,y)=\beta \log\frac{\pi_\theta(y|x)}{\pi_{\theta^\prime}(y|x)}+\beta \log Z(x). 
\end{equation}

However, the closed-form formula contains a partition function $Z(x)$ that is expensive to estimate in practice, because the function requires calculating the expectation of all possible responses.

To address this issue, we identify a critical condition: when the sum of assigned response weights within a prompt group equals zero, $\sum_{i=1}^k w_i=0$, the partition function becomes invariant across responses: 
\begin{equation}
\sum_{i=1}^k w_iR_\theta(x,y_i)= \sum_{i=1}^k w_i\beta\log
\frac{\pi_\theta(y_i|x)}{\pi_{\theta^\prime}(y_i|x)}.    
\end{equation}

\subsection{Method} \label{section:method}
Build on this insight, we propose Group Variance Policy Optimization (GVPO), whose gradient weight $w_i$ is the difference between the central distance of implicit rewards-which derive from policy $\pi_{\theta}$ and policy $\pi_{\theta^\prime}$-and that of actual rewards. Formally, GVPO's gradient $\nabla_\theta\mathcal{L}_{\text{GVPO}}(\theta) =$
\begin{equation}
\begin{aligned}
\label{equation:sample_gvpo}
-\beta\sum_{x, \{y_i\}} \sum_{i=1}^k \Big [(R(x,y_i)-\overline{R})-\beta(\log\frac{\pi_\theta(y_i|x)}{\pi_{\theta^\prime}(y_i|x)} -\overline{\log\frac{\pi_\theta}{\pi_{\theta^\prime}}})\Big]
\\ \nabla_\theta\log \pi_\theta(y_i|x)
\end{aligned}
\end{equation}
where the group average actual reward $\overline{R}=\frac{1}{k}\sum_{i=1}^k R(x,y_i)$, and the average implicit reward $\overline{\log\frac{\pi_\theta}{\pi_{\theta^\prime}}}=\frac{1}{k}\sum_{i=1}^k\log\frac{\pi_\theta(y_i|x)}{\pi_{\theta^\prime}(y_i|x)}$. We note that GVPO's gradient satisfies $\sum_{i=1}^k w_i=0$. 

Algorithm~\ref{alg:gvpo} presents the pseudocode of our proposed algorithm, where we set $\pi_{\theta^\prime}$ to $\pi_{\theta_{old}}$.

We demonstrate that GVPO's objective carries physically meaningful interpretations: 
\begin{equation}
\resizebox{\hsize}{!}{$
\begin{aligned}
\nonumber
&\nabla_\theta\mathcal{L}_{\text{GVPO}}(\theta) \\
=&-\sum_{x, \{y_i\} } \sum_{i=1}^k [(R(x,y_i)-\overline{R})-(R_\theta(x,y_i)-\overline{R_\theta})] \nabla_\theta \beta\log \pi_\theta(y_i|x) \\
=&-\sum_{x, \{y_i\}} \sum_{i=1}^k [(R(x,y_i)-\overline{R})-(R_\theta(x,y_i)-\overline{R_\theta})] \nabla_\theta R_\theta(x,y_i) \\
=&-\sum_{x, \{y_i\} } \sum_{i=1}^k [(R(x,y_i)-\overline{R})-(R_\theta(x,y_i)-\overline{R_\theta})] \nabla_\theta(R_\theta(x,y_i)-\overline{R_\theta}) \\
=& \frac{1}{2}\nabla_\theta\sum_{x, \{y_i\} } \sum_{i=1}^k [(R_\theta(x,y_i)-\overline{R_\theta})-(R(x,y_i)-\overline{R})]^2 \\
\end{aligned}
$}
\end{equation}
\kc{The first and second steps hold because $\beta\log Z(x)$ can cancel out. }
The second step holds because $\sum_{i=1}^k w_i \nabla_\theta \overline{R} =0$. The third step holds because $\nabla_xf(x)^2=2f(x)\nabla_xf(x)$.

Essentially, we have established that GVPO's gradient mathematically equals that of a mean squared error loss measuring the discrepancy between implicit and actual reward central distances. Intuitively, when implicit rewards equal actual rewards or with a constant group shift, the GVPO's loss is minimized. This interpretation also implies that the response with higher actual rewards in a group is also encouraged to have higher implicit rewards, indicating higher $\log\frac{\pi_\theta(y_i|x)}{\pi_{\theta^\prime}(y_i|x)}$.

\kc{Furthermore, by rearranging the mean squared error loss, we can derive a variance-based formulation, which represents the "\textbf{Variance}" term in the name GVPO:}
\[
 \frac{1}{2}\nabla_\theta\sum_{x, \{y_i\} } \sum_{i=1}^k [(R_\theta(x,y_i)-R(x,y_i))-\overline{(R_\theta-R)}]^2
\]

\begin{algorithm}
\caption{Group Variance Policy Optimization}
\label{alg:gvpo}
\begin{algorithmic}[1]
\REQUIRE initial policy $\pi_{\theta}$; prompt distribution $\mathcal{D}$; hyperparameter $\beta$
    \FOR{step $= 1, \dots, n$}
        \STATE Sample a batch $\mathcal{D}_b$ from $\mathcal{D}$
        \STATE Update the old policy model $\pi_{\theta_{old}} \gets \pi_{\theta}$
        \STATE Sample $k$ responses $\{y_i\}_{i=1}^{k} \sim \pi_{s}(\cdot|x)$ for each prompt $x \in \mathcal{D}_b$
        \STATE Compute rewards $\{R(x,y_i)\}_{i=1}^{k}$ for every sampled response $y_i$ and prompt $x$
        \kc{\STATE Iteratively update policy $\pi_{\theta}$ by minimizing the GVPO loss (Equation~\ref{equation:sample_gvpo}, setting $\pi_{\theta^\prime}=\pi_{\theta_{old}}$)}
    \ENDFOR
\STATE \textbf{Return} $\pi_{\theta}$
\end{algorithmic}
\end{algorithm}

\subsection{Theoretical Guarantee}

We show that GVPO has a unique optimal solution, and this unique optimal solution is exactly the optimal solution of reward maximization with KL constraint (Equation \ref{equation:optimal}). 
\begin{theorem}
\label{theorem:gvpo}
The unique optimal policy that minimizes $\hat{\mathcal{L}}_{\text{GVPO}}(\theta)$, defined as
\begin{equation}
\label{equation:gvpo}
\begin{aligned}
\hat{\mathcal{L}}_{\text{GVPO}}(\theta)=
\mathbb{E}_{x \sim \mathcal{D}} \mathbb{E}_{y \sim \pi_{s}(\cdot|x)}\Big[(R_\theta(x,y)-
\mathbb{E}_{y \sim \pi_{s}}R_\theta(x,y))
\\-(R(x,y)-\mathbb{E}_{y \sim \pi_{s}}R(x,y))\Big]^2
\end{aligned}
\end{equation}
, is given by $\pi_\theta (y|x)=\pi^* (y|x)= \frac{1}{Z(x)}\pi_{\theta^\prime}(y|x)e^{R(x,y)/\beta}$ 
for $\pi_s=\pi_{\theta^\prime}$.
\end{theorem}

\kc{We prove the theorem by establishing both necessity and sufficiency, provided in Appendix~\ref{app:proof1}.}

Theorem~\ref{theorem:gvpo} implies that the unique global minimizer of $\hat{\mathcal{L}}_{\text{GVPO}}(\theta)$ also solves the KL-constrained reward maximization problem, maximizing the expected reward while remaining close to a reference policy. The uniqueness of this minimizer means that no other minimum can attain the maximum value of the constrained reward objective, which provides a well-defined optimization target in practice and eliminates ambiguity among competing candidate solutions.

\begin{theorem}
\label{corollary:pi_s}
The Theorem~\ref{theorem:gvpo} also holds for any sampling distribution $\pi_s$ satisfying $\forall x,\{y|\pi_{\theta^\prime}(y|x)>0\}\subseteq\{y|\pi_s(y|x)>0\}.$
\end{theorem}

Beyond the conventional practice of sampling from the reference policy ($\pi_s=\pi_{\theta^\prime}$), GVPO retains the theoretical guarantee of a unique optimal solution under any sampling distribution that satisfies a mild condition. This condition is readily met by any policy $\pi$ where $\pi(y|x)>0$, a criterion inherently fulfilled by contemporary LLM policies utilizing softmax decoding.

The Theorem~\ref{corollary:pi_s} of GVPO opens a new methodological avenue for off-policy LLM post-training. Prior off-policy methods have relied heavily on importance sampling \cite{tokdar2010importance}, which suffers from two key limitations:
(1) when $\pi_\theta$ diverges substantially from $\pi_s$, the importance weight $\frac{\pi_\theta}{\pi_s}$ can become either excessively large or vanishingly small, destabilizing training; and
(2) when sampling involves heuristic or non-parametric components, $\pi_s$ becomes intractable to compute, thereby prohibiting techniques such as experience replay \cite{fedus2020revisiting} in modern LLM post-training. In contrast, GVPO supports highly flexible off-policy sampling strategies while maintaining strong theoretical guarantees, enabling more robust and practical post-training paradigms for large language models.

\begin{theorem}
\label{theorem:gvpo_alg}
The $n$-step online algorithm, which uses $\hat{\mathcal{L}}_{\text{GVPO}}(\theta_t)$ to iteratively update the initial policy $\pi_{\theta_0}$ by setting $\pi_{\theta'} = \pi_{\theta_{t-1}}$ at each step $t = 1, \dots, n$, maximizes the objective:
\begin{equation}
\label{equation:gvpo_alg}
\mathbb{E}_{x\sim\mathcal{D}, y\sim\pi_\theta(y|x)}[R(x,y)] - \frac{\beta}{n}\mathbb{D}_{\text{KL}}[\pi_\theta(y|x) \| \pi_{\theta_0}(y|x)].
\end{equation}
\end{theorem}

Theorem \ref{theorem:gvpo_alg} characterizes the optimization objective of the $n$-step online GVPO procedure. At each iteration, the policy is updated relative to its immediate predecessor with a KL regularization coefficient $\beta$, while the overall procedure is equivalent to optimizing a reward objective regularized by the KL divergence from the initial policy $\pi_{\theta_0}$ with an effective coefficient $\beta/n$. A larger $\beta$ promotes more stable local updates at each step, while the decreasing effective coefficient $\beta/n$ as $n$ increases allows the final policy to deviate further from the initial policy, providing more room for optimization. The detailed proof is provided in Appendix~\ref{app:proof2}.

\subsection{Comparisons with DPO}
In this section, we discuss the similarities between GVPO and DPO, their key design principles, and the theoretical advantages offered by GVPO.

We begin by analyzing the foundational commonality between GVPO and DPO: both methods integrate the closed-form solution to the KL-constraint reward maximization problem into their training objectives. This integration establishes a direct relationship between the learned policy $\pi_\theta$ and the implicit reward function $R_\theta$, yielding two key advantages:
\begin{itemize}[leftmargin=0.5cm]
    \item It ensures an optimization process that inherently respects the KL divergence constraint, thereby preventing excessive deviation of the policy $\pi_\theta$ from the reference policy $\pi_{ref}$.
    \item It reduces the joint optimization over policies and rewards to a simpler problem focused solely on rewards. The latter is more tractable, as it requires only aligning the implicit rewards $R_\theta(x, y)$ with the true reward function $R(x, y)$.
\end{itemize}

The closed-form solution reveals two critical design principles that also distinguish DPO and GVPO:

\begin{enumerate} [leftmargin=0.5cm]
    \item \textbf{Computational Tractability}: The method must avoid intractable terms such as the partition function $Z(x)$. For instance, a naive loss $\mathcal{L} = \sum (R_\theta(x, y) - R(x, y))^2$ fails because $R_\theta(x, y)$ implicitly depends on $Z(x)$, which is computationally infeasible to estimate. DPO circumvents this by adopting the Bradley-Terry preference model, where $Z(x)$ cancels out in pairwise comparisons. GVPO proposes a novel \textit{zero-sum property} across groups of responses, enabling cancellation of $Z(x)$ in broader multi-sample scenarios.
    
    \item \textbf{Alignment with Desired Optimality}: The loss function must enforce meaningful training. For example, minimizing $\mathcal{L} = \sum \left(\beta \log \frac{\pi_\theta(x, y)}{\pi_{ref}(x, y)} - R(x, y)\right)^2$ yields a suboptimal solution $R_\theta(x, y) = R(x, y) + \beta \log Z(x)$, which deviates from the true reward $R(x, y)$. A well-designed objective must avoid such misalignment. The method should adapt to available supervision. DPO leverages pairwise preference data without explicit rewards, while GVPO generalizes to group-wise responses with reward signals.
\end{enumerate}

\kc{Moreover, GVPO demonstrates stronger theoretical robustness compared to DPO:}
\begin{itemize}[leftmargin=0.5cm]
    \item Prior work \cite{bong2022generalized,hong2024energybasedpreferencemodeloffers} highlights that DPO may fail to align with the optimal policy for the KL-constrained reward maximization problem, because of the inherent flaw of Bradley-Terry model \cite{wu2022asymptoticcomparisonidentifyingconstraints}. This arises because the DPO loss admits multiple minimizers, and its correlation with the true reward objective can diminish during training \cite{tang2024generalized}.
    \item In contrast, as formalized in Theorem~\ref{theorem:gvpo} and Theorem~\ref{corollary:pi_s}, GVPO guarantees that its loss function is aligned with the original constrained optimization problem, ensuring its global optimum. This theoretical robustness positions GVPO as a more reliable method for policy optimization in practice.
\end{itemize}

\subsection{Comparisons with GRPO and Policy Gradient Methods} \label{sec:reg}
Seeing the forest for the trees, we compare GVPO not only with GRPO but also with the broader family of policy gradient-based RL methods.

\textbf{Structural similarities}. We begin by examining their superficial structural similarities. We then establish a regularization perspective of GVPO, decomposing it into distinct regularization terms.
For simplicity, we assume \(\beta = 1\) without loss of generality. Then
\begin{equation}
\label{equation:gvpovsgrpo}
\begin{aligned}
\hat{\mathcal{L}}_{\text{GVPO}}(\theta)\overset{\nabla_\theta}{=}&\mathop{\mathbb{E}}_{x \sim \mathcal{D},y \sim \pi_{s}(\cdot|x)} \left[ (R_\theta(x,y) - \mathbb{E}_y R_\theta(x,y))^2 \right. \\
&\quad \left. - 2(R(x,y) - \mathbb{E}_y R(x,y))R_\theta(x,y) \right] \\
\overset{\nabla_\theta}{=}&\, \mathbb{E}_{x,y} \left[ \mathrm{Var}(\log\pi_\theta) - 2\mathrm{Cov}(\log\pi_\theta,\log\pi_{\theta^\prime}) \right. \\
&\quad \left. - 2(R(x,y) - \mathbb{E}_y R(x,y))\log\pi_\theta(y|x) \right] \\
=& -2\mathbb{E}_{x,y} \left[ (R(x,y) - \mathbb{E}_y R(x,y))\log\pi_\theta(y|x) \right. \\
&\quad \left. + \mathrm{Cov}(\log\pi_\theta,\log\pi_{\theta^\prime}) - 0.5\mathrm{Var}(\log\pi_\theta) \right]
\end{aligned}
\end{equation}
where $Var(\log\pi_\theta)=(\log\pi_\theta(y|x)-\mathbb{E}_y\log\pi_\theta(y|x))^2$ and the covariance term $Cov(\log\pi_\theta,\log\pi_{\theta^\prime})=(\log\pi_\theta(y|x)-\mathbb{E}_y\log\pi_\theta(y|x))(\log\pi_{\theta^\prime}(y|x)-\mathbb{E}_y\log\pi_{\theta^\prime}(y|x))$. 

As shown in Equation~\ref{equation:gvpovsgrpo}, the GVPO loss can be decomposed into three components: \((R(x,y) - \mathbb{E}_y R(x,y))\log\pi_\theta(y|x)\), \(Cov(\log\pi_\theta, \log\pi_{\theta^\prime})\) and \(Var(\log\pi_\theta)\).
\begin{itemize}[leftmargin=0.5cm]
\item the term \((R(x,y) - \mathbb{E}_y R(x,y))\log\pi_\theta(y|x)\) encourages advantage maximization. GRPO directly optimizes the advantage by standardizing reward scores across sampled responses. In contrast, GVPO does not include standard deviation normalization, not as a heuristic design choice but as a consequence of its theoretical formulation. Prior work \cite{liu2025understanding} has also shown that reward normalization based on the standard deviation can introduce bias by conflating prompt-level difficulty with the underlying reward signal.

\item the term \(Cov(\log\pi_\theta, \log\pi_{\theta^\prime})\) serves to constrain deviations of the policy \(\pi_\theta\) from policy $\pi_{\theta^\prime}$, corresponding to \(\mathbb{D}_{\text{KL}}[\pi_\theta||\pi_{\theta^\prime}]\).  When we set $\pi_{\theta^\prime}=\pi_\text{ref}$, this term prevent deviating from a fixed policy $\pi_\text{ref}$. When we set $\pi_{\theta^\prime}=\pi_{\theta_{\text{old}}}$, this term essentially aligns with the trust-region constraint \cite{pmlr-v37-schulman15}, that ensures robustness between policy updates.

\item the term \(Var(\log\pi_\theta)\) is considered alongside the entropy regularization term \(-\mathbb{E}_y\log\pi(y|x)\) \cite{ahmed2019understanding}, which strikes a balance between exploration and exploitation.
\begin{itemize}[leftmargin=0.5cm]
\item Increasing entropy encourages diversity by driving the policy toward a uniform distribution, but risks suppressing the likelihood of high-quality responses. Conversely, reducing entropy concentrates probability mass on a narrow set of outputs, diminishing diversity and potentially inducing entropy collapse. Consequently, entropy regularization proves highly sensitive to its coefficient, complicating practical implementation.
\item In contrast, \(Var(\log\pi_\theta)\) allows the policy to assign zero probability to some responses while maintaining comparable probabilities among the remaining responses, without requiring ad-hoc tuning. Furthermore, the term $R - \overline{R}$ assigns positive advantages to favorable responses and negative advantages to undesirable ones, encouraging the latter toward zero probability while maintaining comparable probabilities among favorable responses.

\item For instance, consider a toy example involving generation over the tokens $a, b, c, d, e$, where $\pi_\theta(a) = \pi_\theta(b) = 0$ and $\pi_\theta(c) = \pi_\theta(d) = \pi_\theta(e) = 1/3$. In this case, $Var(\log(\pi_\theta)) = 0$, indicating minimum variance and thus no penalty on this distribution. In contrast, entropy regularization pushes the distribution toward either the fully uniform or one-hot solution, depending on whether higher or lower entropy is encouraged.

\end{itemize}
\end{itemize}
\textbf{Deeper Connections}. In addition to structural similarities, we now highlight their key theoretical and practical distinctions. Modern policy gradient methods \cite{pmlr-v37-schulman15,schulman2017proximal,shao2024deepseekmath} optimize the expected reward under the current policy $\pi_\theta$ while constraining updates to avoid excessive deviation from the previous policy $\pi_{\theta_{\text{old}}}$. This is typically achieved by optimizing a  objective that combines the reward $R(x,y)$ and a KL-divergence penalty term $\mathbb{D}_{\text{KL}}[\pi_\theta||\pi_{\theta_{\text{old}}}]$, yielding the gradient expression:

\kc{
\begin{equation} 
\begin{aligned}
\label{equation:pg}
&\nabla_\theta[\mathop{\mathbb{E}}_{x,y\sim\pi_\theta(y|x)}[R(x,y)]-\mathbb{D}_{\text{KL}}[\pi_\theta||\pi_{\theta_{\text{old}}}]] \\
= &\nabla_\theta\mathbb{E}_{x} \sum_y \pi_\theta(y|x)(R(x,y)-\log\frac{\pi_\theta(y|x)}{\pi_{\theta_{\text{old}}}(y|x)}) \\
= &\mathop{\mathbb{E}}_{x,y\sim\pi_\theta(y|x)}(R(x,y)-\log\frac{\pi_\theta(y|x)}{\pi_{\theta_{\text{old}}}(y|x)}-1)  \nabla_\theta \log \pi_\theta(y|x)
\end{aligned}
\end{equation}
}
However, estimating this expectation requires on-policy sampling from $\pi_\theta(y|x)$, leading to low sample efficiency---a well-documented limitation of policy gradient methods. Reusing stale samples from prior policies introduces bias, degrading optimization stability and final performance.  

To mitigate this, prior works~\cite{pmlr-v37-schulman15,schulman2017proximal,shao2024deepseekmath} employ importance sampling, rewriting Equation~\ref{equation:pg} as:
\begin{equation} \nonumber
\resizebox{\hsize}{!}{$
\mathop{\mathbb{E}}\limits_{x,y\sim\pi_{\theta_{\text{old}}}}\frac{\pi_\theta(y|x)}{\pi_{\theta_{\text{old}}}(y|x)}\left(R(x,y)-\log\frac{\pi_\theta(y|x)}{\pi_{\theta_{\text{old}}}(y|x)} -1 \right) \nabla_\theta \log \pi_\theta(y|x).
$}
\end{equation}
This allows off-policy gradient estimation using samples from $\pi_{\theta_{\text{old}}}$. However, $\frac{\pi_\theta(y|x)}{\pi_{\theta_{\text{old}}}(y|x)}$ becomes unstable when $\pi_\theta$ deviates significantly from $\pi_{\theta_{\text{old}}}$, risking gradient explosion. Heuristics like gradient clipping \cite{schulman2017proximal} address this at the cost of biased gradient estimates, undermining theoretical guarantees.  

GVPO circumvents these issues because it does not necessitate on-policy sampling in the first place. By rearranging Equation~\ref{equation:pg}, we observe that the policy gradient can be expressed as:
\begin{equation}\nonumber
\begin{aligned}
\mathop{\mathbb{E}}_{x,y\sim\pi_\theta(y|x)}\Big[R(x,y)-\log\frac{\pi_\theta(y|x)}{\pi_{\theta_{\text{old}}}(y|x)} - \mathbb{E}_{y\sim\pi_\theta(y|x)}\big(R(x,y)\\-\log\frac{\pi_\theta(y|x)}{\pi_{\theta_{\text{old}}}(y|x)}\big)\Big] \nabla_\theta \log \pi_\theta(y|x),
\end{aligned}
\end{equation}
where the baseline term (subtracted expectation) arises because $\mathbb{E}_{y\sim\pi_\theta} [c \nabla_\theta \log \pi_\theta(y|x)] = \nabla_\theta c = 0$ for any constant $c$. Crucially, the GVPO gradient generalizes this structure, $\nabla_\theta\hat{\mathcal{L}}_{\text{GVPO}}(\theta) = $
\begin{equation}\nonumber
\begin{aligned}
\mathop{\mathbb{E}}_{x,y\sim\pi_s(y|x)}\Big[R(x,y)-\log\frac{\pi_\theta(y|x)}{\pi_{\theta_{\text{old}}}(y|x)} - \mathbb{E}_{y\sim\pi_s(y|x)}\big(R(x,y)\\-\log\frac{\pi_\theta(y|x)}{\pi_{\theta_{\text{old}}}(y|x)}\big)\Big] \nabla_\theta \log \pi_\theta(y|x),
\end{aligned}
\end{equation}
This reveals that classical policy gradient under trust-region constraint is a special case of GVPO gradient with $\pi_s = \pi_\theta$. As proven in Theorem~\ref{theorem:gvpo} and Theorem~\ref{corollary:pi_s}, GVPO retains the same optimal solution as the policy gradient method while decoupling the sampling distribution $\pi_s$ from the learned policy $\pi_\theta$. GVPO’s decoupling addresses two critical limitations:  
\begin{enumerate} [leftmargin=0.5cm]
    \item \textbf{Sample Efficiency}: Unlike on-policy methods \cite{singh2000convergence, williams1992simple}, GVPO supports off-policy training with reusable or mixed data (e.g., expert demonstrations or historical policies).  
    \item \textbf{Stability}: By avoiding importance sampling weights $\frac{\pi_\theta}{\pi_{\theta_{\text{old}}}}$, GVPO eliminates gradient explosion risks without biased clipping, thereby promoting training reliability in practice.
\end{enumerate}
By synergizing these advantages, GVPO emerges as a competitive online reinforcement learning algorithm capable of leveraging diverse data sources and preserving optimality guarantee—a combination previously unattained in prior policy gradient methods.

\input{}
\clearpage

\begin{figure*}[t]
\centering
\includegraphics[width=0.9\textwidth]{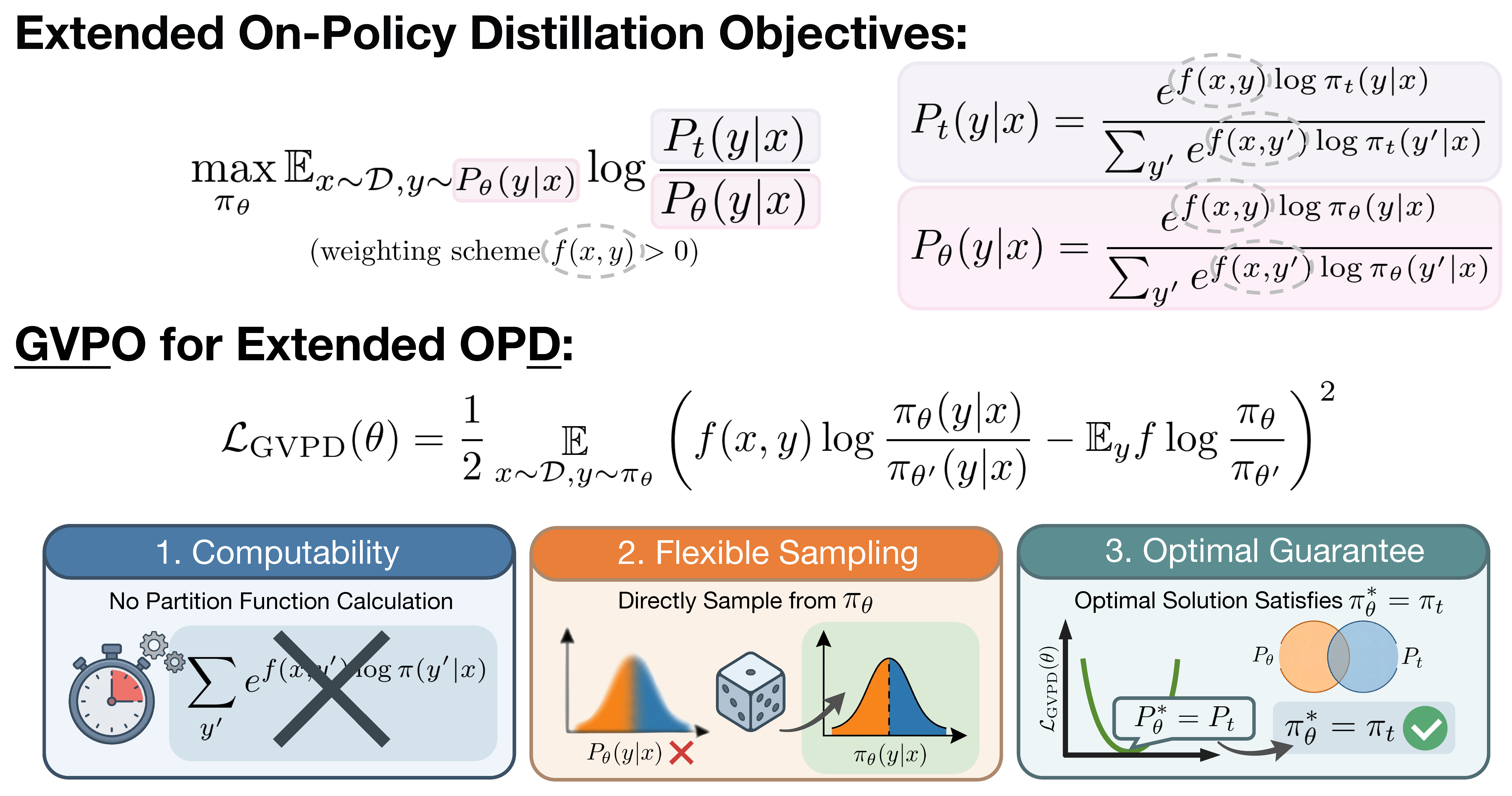}
% \vspace*{-5mm}
\caption{\textbf{GVPO for extended on-policy distillation objectives.} An extended OPD objective assigns a weighting scheme $f(x,y)$ to each $(x,y)$ and minimizes the reverse KL divergence between the induced distributions $P_\theta$ and $P_t$. GVPO optimizes it without computing partition functions, enabling sampling from $\pi_\theta$ beyond $P_\theta$ while ensuring the desired optimum $\pi_\theta^*=\pi_t$.
}
\label{fig:gvpd}
\vspace{-5mm}
\end{figure*}

\section{GVPO for On-Policy Distillation}
\subsection{Method}
On-policy distillation  of large language models can be formulated as minimizing the reverse KL divergence between the student policy $\pi_\theta$ and a teacher policy $\pi_t$:
\begin{equation}
    \min_{\pi_\theta}\mathbb{D}_{KL}[\pi_\theta(y|x)||\pi_t(y|x)].
\end{equation}
Equivalently, this objective can be written as maximizing
\begin{equation} \label{equation:opd}
    \max_{\pi_\theta}\mathbb{E}_{x\sim\mathcal{D},y\sim\pi_\theta(y|x)} \log\frac{\pi_t(y|x)}{\pi_\theta(y|x)}.
\end{equation}
To connect OPD with GVPO, we introduce an arbitrary reference policy $\pi_{\theta'}$. For any response $y$, we have the following decomposition:
\begin{equation}
    \log\frac{\pi_t(y|x)}{\pi_\theta(y|x)} = \log\frac{\pi_t(y|x)}{\pi_{\theta^\prime}(y|x)} - \log\frac{\pi_\theta(y|x)}{\pi_{\theta^\prime}(y|x)}.
\end{equation}

We define the reward as $R(x,y) = \log\frac{\pi_t(y|x)}{\pi_{\theta^\prime}(y|x)}$. Importantly, $R(x,y)$ is independent of the current student policy $\pi_\theta$. Substituting this definition into Equation~\ref{equation:opd} yields the equivalent KL-regularized reinforcement learning objective with $\beta=1$:
\begin{equation}
    \max_{\pi_{\theta}} \mathbb{E}_{x\sim\mathcal{D},y\sim\pi_\theta(y|x)}R(x,y)-\beta\log\frac{\pi_\theta(y|x)}{\pi_{\theta^\prime}(y|x)}.
\end{equation}

Consequently, GVPO can directly optimize the on-policy distillation objective. We denote the resulting loss as GVPD:
\begin{equation}
    \begin{aligned}
    &\mathcal{L}_{\text{GVPD}}(\theta)  \\
    &=\frac{1}{2}\sum_{x, \{y_i\} } \sum_{i=1}^k [(R_\theta(x,y_i)-\overline{R_\theta})-(R(x,y_i)-\overline{R})]^2  \\
    &=\frac{1}{2}\sum_{x, \{y_i\} } \sum_{i=1}^k \Big(\log\frac{\pi_\theta(y|x)}{\pi_t(y|x)}-\overline{\log\frac{\pi_\theta}{\pi_t}}\Big)^2.  
    \end{aligned}
\end{equation}
The second equality follows from $R_\theta(x,y_i)=\log\frac{\pi_\theta(y_i|x)}{\pi_{\theta^\prime}(y_i|x)}$ and $R(x,y_i)=\log\frac{\pi_t(y_i|x)}{\pi_{\theta^\prime}(y_i|x)}$. Thus, the reference policy $\pi_{\theta^\prime}$ introduced during the derivation cancels out, yielding the final objective above.

The above formulation allows the advantages of GVPO to be directly inherited by on-policy distillation. First, the objective admits the teacher policy $\pi_t$ as its exact optimum without any approximation or bias. Second, GVPO accommodates flexible sampling distributions beyond strictly on-policy sampling without requiring importance-sampling corrections. 

Furthermore, GVPO operates at the sequence level rather than requiring token-level alignment between the student and teacher. Consequently, the student and teacher need not share the same vocabulary or tokenization scheme. This substantially broadens its applicability, enabling distillation between models from different model families or series whose tokenizers and vocabularies are incompatible.

\subsection{Extended On-Policy Distillation}
We study a broad family of extended OPD objectives:
\begin{equation} \label{equation:extended_opd}
    \max_{\pi_\theta}\mathbb{E}_{x\sim\mathcal{D},y\sim P_\theta(y|x)} \log\frac{P_t(y|x)}{P_\theta(y|x)}.
\end{equation},
where each response $(x,y)$ is assigned a positive weight $f(x,y)>0$. Specifically, we define $P_t(y|x)$ and $P_\theta(y|x)$ as
\begin{equation}
\begin{aligned}
    &P_t(y|x) = \frac{ e^{f(x,y)\log\pi_t(y|x)} }{ \sum_{y'}e^{f(x,y')\log\pi_t(y'|x)} },\\
    &P_\theta(y|x) = \frac{ e^{f(x,y)\log\pi_\theta(y|x)} }{ \sum_{y'}e^{f(x,y')\log\pi_\theta(y'|x)} }.
\end{aligned}
\end{equation}

Both $P_t(y|x)$ and $P_\theta(y|x)$ are probability distributions induced by the weighting scheme $f(x,y)$. Consequently, the objective in Equation~\ref{equation:extended_opd} is equivalent to minimizing the reverse KL divergence from $P_t(y|x)$ to $P_\theta(y|x)$. When $f(x,y)=1$ for all $(x,y)$, the extended OPD objective reduces to the standard OPD objective (Equation \ref{equation:opd}).

We first establish the following result.
\begin{theorem} \label{theorem:opd1}
If $P_\theta(y\mid x)=P_t(y\mid x)$ for all $x,y$, then
$\pi_\theta(y\mid x)=\pi_t(y\mid x)$ for all $x,y$.
\end{theorem}
\begin{proof}
For each $x$, define the partition functions $Z_t(x)=\sum_{y'}e^{f(x,y')\log\pi_t(y'|x)}$ and $Z_\theta(x)=\sum_{y'}e^{f(x,y')\log\pi_\theta(y'|x)}$.

By the assumption $P_\theta(y\mid x)=P_t(y\mid x)$, we have
\[
\pi_\theta(y\mid x)
=
\pi_t(y\mid x)
\left(\frac{Z_\theta(x)}{Z_t(x)}\right)^{1/f(x,y)}.
\]
Since $\sum_y\pi_\theta(y\mid x)=1$, it follows that
\[
1
=
\sum_y\pi_t(y\mid x)
\left(\frac{Z_\theta(x)}{Z_t(x)}\right)^{1/f(x,y)}.
\]
\[
\text{Let } C_x=\log\frac{Z_\theta(x)}{Z_t(x)}
\quad\text{and}\quad
H_x(C)=\sum_y\pi_t(y\mid x)e^{C/f(x,y)}.
\]
Then $H_x(C_x)=1$. Moreover, since $f(x,y)>0$,
\[
H_x'(C)
=
\sum_y
\frac{\pi_t(y\mid x)}{f(x,y)}
e^{C/f(x,y)}
>0.
\]
Thus $H_x$ is strictly increasing. Since
$H_x(0)=\sum_y\pi_t(y\mid x)=1$,
the equation $H_x(C_x)=1$ has the unique solution $C_x=0$. Therefore,
$Z_\theta(x)=Z_t(x)$,
and consequently
\[
\pi_\theta(y\mid x)
=
\pi_t(y\mid x)
\left(\frac{Z_\theta(x)}{Z_t(x)}\right)^{1/f(x,y)}
=
\pi_t(y\mid x), \quad \forall x,y.
\]
\end{proof}

Theorem~\ref{theorem:opd1} establishes that optimizing the objective in Equation~\ref{equation:extended_opd} is consistent with the OPD objective, namely, recovering $\pi_\theta=\pi_t$. However, directly optimizing this objective presents two practical challenges. First, evaluating the partition functions $Z_t(x)$ and $Z_\theta(x)$ requires summing over the entire response space, making $P_t$ and $P_\theta$ difficult to estimate in practice. Second, to the best of our knowledge, on-policy sampling from $P_\theta$ is not feasible under the standard autoregressive decoding procedure of LLMs. In particular, sampling from $P_\theta$ would require access to all response-level probabilities $P_\theta(y\mid x)$ in advance, which is computationally intractable.

We next show that GVPO can optimize the extended reverse KL objective:
\begin{equation}
    \begin{aligned}
    &\mathcal{L}_{\text{GVPD}}(\theta)  \\
    &=\frac{1}{2}\sum_{x, \{y_i\} } \sum_{i=1}^k \Big(\log\frac{P_\theta(y|x)}{P_t(y|x)}-\overline{\log\frac{P_\theta}{P_t}}\Big)^2 \\
    &=\frac{1}{2}\sum_{x, \{y_i\} } \sum_{i=1}^k \Big(f(x,y)\log\frac{\pi_\theta(y|x)}{\pi_t(y|x)}-\overline{f\log\frac{\pi_\theta}{\pi_t}}\Big)^2.  
    \end{aligned}
\end{equation}
The second step holds because $\log\frac{P_\theta(y|x)}{P_t(y|x)}=f(x,y)\log\frac{\pi_\theta(y|x)}{\pi_t(y|x)}-\log\frac{Z_\theta(x)}{Z_t(x)}$ and the partition term $\log\frac{Z_\theta(x)}{Z_t(x)}$ cancels when subtracting the group average. This cancellation provides another manifestation of the key insight underlying GVPO:  \textit{when the sum of assigned response weights within a prompt group equals zero, the partition function becomes invariant across responses.}

Furthermore, GVPO supports sampling from a broad class of distributions, eliminating the need to sample directly from $P_\theta$ when optimizing the extended OPD objective. In particular, we can sample responses from $\pi_\theta$ or from the previous policy $\pi_{\theta_{\mathrm{old}}}$ while retaining the theoretical guarantee established above. Thus, GVPO provides a practical way to optimize the extended OPD objective while avoiding both the intractable partition-function computation and the infeasible on-policy sampling procedure associated with $P_\theta$.

\subsection{Example} \label{sec:opd_example}
To motivate the use of the extended OPD objective, we provide a concrete example of a weighting scheme $f$.

In the standard OPD objective, a token $y_t$ is generated on-policy from $\pi_\theta(\cdot \mid x,y_{<t})$. Compared with the teacher policy $\pi_t$, the source policy $\pi_\theta$ may generally assign a higher probability to the generated token $y_t$. That is, it is possible that
$\pi_\theta(y|x_t,y_{<t})>\pi_t(y_t|x,y_{<t})$ and consequently, $\log\frac{\pi_\theta(y_t|x,y_{<t})}{\pi_t(y_t|x,y_{<t})}<0$.

This effect can introduce a length bias into the optimization. As the response becomes longer, the cumulative log-ratio $\log\frac{\pi_\theta(y|x)}{\pi_t(y|x)}$ can keep decreasing. Consequently, longer responses may receive smaller gradient weights. This can create a feedback loop in which the average response length of $\pi_\theta$ continually decreases during optimization, eventually leading the model to produce very short responses that lack sufficient information as illustrated in Figure~\ref{fig:length_plot}.

To mitigate this limitation of the standard OPD objective, we consider the weighting scheme
$f(x,y)=\frac{1}{|y|^\alpha}$ where $|y|$ denotes the response length and $\alpha$ is a hyperparameter controlling the degree of length normalization. Intuitively, this weighting scheme counteracts the inherent length bias of the standard OPD objective by explicitly normalizing the contribution of a response according to its length.

More broadly, GVPO enables the heuristic design of a broader family of extended OPD objectives. In particular, one can design weighting schemes that suppress or promote specific undesirable or desirable response patterns during distillation, while retaining the theoretical guarantee that the teacher policy $\pi_t$ is its exact optimization destination.

\begin{figure}[b]
\vspace{-5mm}
\centering
\includegraphics[width=0.4\textwidth]{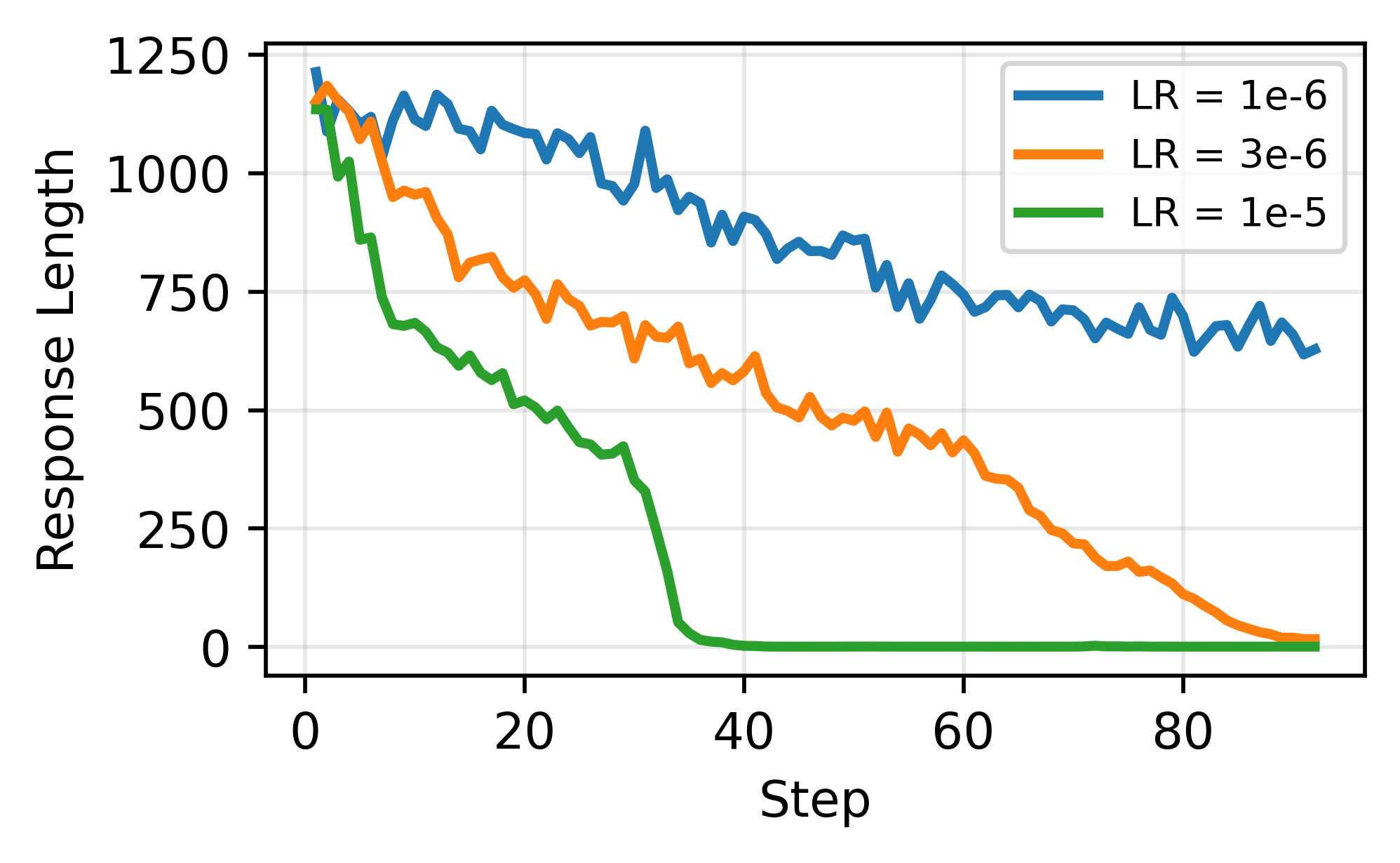}
\vspace*{-3mm}
\caption{
\textbf{OPD’s response length continually decreases under policy-gradient training across different learning rates.}
}
\label{fig:length_plot}
\vspace{-5mm}
\end{figure}
\newpage
\begin{table*}[htbp]
\centering
\caption{Performance comparison of post-training algorithms on Qwen2.5-Math-7B across mathematical reasoning benchmarks, along with key theoretical properties, including whether each method provides the optimal-solution guarantee and whether it avoids the use of importance sampling for off-policy training.}
\label{tab:results}
\begin{tabular}{lcccccccc}
\toprule
Algorithm & AIME2024 & AMC & MATH500 & Minerva & Olympiadbench & Average & \begin{tabular}[c]{@{}c@{}}Optimal Solution\\ Guarantee\end{tabular} & \begin{tabular}[c]{@{}c@{}}No IS Weight \\ for Off-policy\end{tabular} \\
\midrule
Qwen2.5-Math-7B & 14.68 & 38.55 & 64.00 & 27.20 & 30.66 & 35.02 & & \\
\quad+GRPO     & 14.79 & 55.42 & 80.00 & 41.17 & 42.07 & 46.69 & $\times$ & $\times$ \\
\quad+Dr.GRPO  & 16.56 & 48.19 & 81.20 & 44.48 & 43.40 & 46.77 & $\times$ & $\times$ \\
\quad+Remax    & 17.19 & 60.24 & 82.00 & 40.44 & 45.19 & 49.01 & $\times$ & $\times$ \\
\quad+Reinforce++ & 16.67 & 54.22 & 80.40 & 43.01 & 41.78 & 47.22 & $\times$ & $\times$ \\
\quad+GSPO     & 18.33 & 61.44 & 81.60 & 40.44 & 44.74 & 49.31 & $\times$ & $\times$ \\
\quad+GMPO     & 19.58 & 50.60 & 82.20 & 38.97 & 46.07 & 47.48 & $\times$ & $\times$ \\
\quad+GPG     & 13.54 & 44.57 & 77.80 & 43.01 & 42.81 & 44.34 & $\times$ & $\times$ \\
\quad+Cispo     & 15.41 & 50.60 & 80.00 & 41.17 & 45.18 & 46.47 & $\times$ & $\times$ \\

\quad+\textbf{GVPO} & \textbf{20.72} & \textbf{62.65} & \textbf{83.80} & \textbf{45.95} & \textbf{46.96} & \textbf{52.02} & \checkmark & \checkmark \\
\bottomrule
\end{tabular}
\vspace{-4.5mm}
\end{table*}

\begin{figure*}[bh!]
  \centering
  \begin{minipage}[b]{0.32\textwidth}
    \centering
    \begin{tikzpicture}
\begin{axis}[
    xlabel=$\beta$,
    ylabel=Accuracy,
    xmode=log,
    log ticks with fixed point,
    xtick={0.01,0.05,0.1,0.2,0.5},
    xmin=0.008, xmax=0.6,
    ymin=0, ymax=1.0,
    grid=both,
    legend style={},
    ticklabel style = {font=\tiny},
    width=5.6cm,
    height=4.5cm,
]

% AIME2024
\addplot+[smooth, black, mark=*] coordinates {
    (0.01, 0.2041)
    (0.05, 0.2156)
    (0.1, 0.2072)
    (0.2, 0.1906)
    (0.5, 0.2010)
};
% \addlegendentry{Dataset A}

% AMC
\addplot+[smooth, black, mark=square*] coordinates {
    (0.01, 0.5542)
    (0.05, 0.5783)
    (0.1, 0.6265)
    (0.2, 0.6385)
    (0.5, 0.5783)
};
% \addlegendentry{Dataset B}

% MATH500
\addplot+[smooth, black, mark=triangle*] coordinates {
    (0.01, 0.8240)
    (0.05, 0.8420)
    (0.1, 0.8380)
    (0.2, 0.82)
    (0.5, 0.8420)
};
% \addlegendentry{Dataset C}

% Minerva
\addplot+[smooth, black, mark=diamond*] coordinates {
    (0.01, 0.4448)
    (0.05, 0.4301)
    (0.1, 0.4595)
    (0.2, 0.4264)
    (0.5, 0.4264)
};
% \addlegendentry{Dataset D}

% Olympiadbench
\addplot+[smooth, black, mark=pentagon*] coordinates {
    (0.01, 0.4370)
    (0.05, 0.4725)
    (0.1, 0.4696)
    (0.2, 0.4559)
    (0.5, 0.4607)
};
% \addlegendentry{Dataset E}

\end{axis}
\end{tikzpicture}
    \caption{Ablation on $\beta$. Each line represents a dataset.}
    \label{fig:beta}
  \end{minipage}
  \hfill
  \begin{minipage}[b]{0.32\textwidth}
    \centering
    \begin{tikzpicture}
  \begin{groupplot}[
    group style={
      group size=1 by 2,
      vertical sep=0.1cm,
      group name=myplots,
      x descriptions at=edge bottom  % 仅最后一个子图显示x轴描述
    },
    xmode=log,
    log basis x=2,
    width=5.6cm, 
    height=3.0cm,
    grid=major,
    xtick={2,4,8,16,32},
    legend columns=2,
    legend style={
      at={([yshift=-15pt]myplots c1r2.south)},
      anchor=north
    },
    ticklabel style = {font=\tiny},
    legend to name=commonlegend
    ]
    
    %%% 第一个子图 (数据集1) %%%
    \nextgroupplot[
      % xtick=\empty,        % 取消x轴刻度
      xlabel=,             % 取消x轴标签
      ymin=0.74, ymax=0.8  % 建议设置相同的y轴范围
    ]
    \addplot+[solid, blue, mark=*] coordinates {
      (2, 0.7500) (4, 0.78) (8, 0.778) (16, 0.784) (32, 0.792)
    };
    \addlegendentry{GVPO}
    
    \addplot+[densely dashed, red, mark=square*] coordinates {
      (2, 0.7420) (4, 0.7520) (8, 0.7600) (16, 0.7620) (32, 0.7660)
    };
    \addlegendentry{GRPO}

    %%% 第二个子图 (数据集2) %%%
    \nextgroupplot[
      ymin=0.1, ymax=0.21  % 保持相同y轴范围
      ,xlabel=$k$
    ]
    \addplot+[solid, blue, mark=*] coordinates {
      (2, 0.1395) (4, 0.1468) (8, 0.1656) (16, 0.1697) (32, 0.202)
    };
    
    \addplot+[densely dashed, red, mark=square*] coordinates {
      (2, 0.1135) (4, 0.1239) (8, 0.1531) (16, 0.1645) (32, 0.1697)
    };

  \end{groupplot}

  % 共用Y轴标签
  \node[rotate=90, anchor=center] at ([xshift=-2mm,yshift=-6mm]myplots c1r1.outer west) {Accuracy};
  
  % 显示共用图例
  \node[anchor=north] at (myplots c1r2.south) {\pgfplotslegendfromname{commonlegend}};
\end{tikzpicture}
    \caption{Ablation on $k$. Blue line: \textcolor{blue}{GVPO}; Red line: \textcolor{red}{GRPO}.}
    \label{fig:k}
  \end{minipage}
  \hfill
  \begin{minipage}[b]{0.32\textwidth}
    \centering
    \begin{tikzpicture}
\begin{axis}[
    xlabel=\#Traning Steps,
    xlabel style = {font=\small},
    ylabel=Accuracy,
    log ticks with fixed point,
    xtick={0,20,40,60,80},
    xmin=-10, xmax=90,
    ymin=0.1, ymax=0.9,
    grid=both,
    legend style={font=\tiny},
    ticklabel style = {font=\tiny},
    width=5.6cm,
    height=4.5cm,
    legend pos = south east,
    legend columns=2, 
]

\addplot+[smooth, blue, mark=o] coordinates {
(0, 0.424)
(10, 0.686)
(20, 0.724)
(30, 0.736)
(40, 0.736)
(50, 0.758)
(60, 0.734)
(70, 0.724)
(80, 0.754)
};
\addlegendentry{4:4}

\addplot+[smooth, red, mark=o] coordinates {
(0, 0.424)
(10, 0.684)
(20, 0.724)
(30, 0.74)
(40, 0.742)
(50, 0.756)
(60, 0.75)
(70, 0.73)
(80, 0.742)
};
\addlegendentry{3:5}

\addplot+[smooth, orange, mark=o] coordinates {
(0, 0.424)
(10, 0.688)
(20, 0.712)
(30, 0.738)
(40, 0.742)
(50, 0.762)
(60, 0.75)
(70, 0.742)
(80, 0.744)
};
\addlegendentry{2:6}

\addplot+[smooth, cyan, mark=o] coordinates {
(0, 0.424)
(10, 0.7)
(20, 0.716)
(30, 0.754)
(40, 0.744)
(50, 0.77)
(60, 0.75)
(70, 0.78)
(80, 0.77)
};
\addlegendentry{1:7}

\addplot+[smooth, olive, mark=o] coordinates {
(0, 0.424)
(10, 0.696)
(20, 0.702)
(30, 0.732)
(40, 0.738)
(50, 0.756)
(60, 0.756)
(70, 0.768)
(80, 0.752)
};
\addlegendentry{0:8}

\end{axis}
\end{tikzpicture}
    \caption{Ablation on $\pi_s$. \#(historical $y$) : \#($y$ from $\pi_{\theta_{\text{old}}}$)}
    \label{fig:pi_s}
  \end{minipage}
\end{figure*}
\section{Post-training Experiments}
\subsection{Experimental Setup}
Following the established experimental setting of GRPO, we conduct a comprehensive evaluation on math reasoning. Specifically, we post-train the Qwen2.5-Math-7B model on Competition Math \cite{hendrycksmath2021} and assess performance on AIME2024 \cite{li2024numinamath}, AMC \cite{li2024numinamath}, Math500 \cite{lightman2023lets}, Minerva \cite{lewkowycz2022solving}, and OlympiadBench \cite{he2024olympiadbench}. For answer verification, we utilize the xVerify framework \cite{chen2025xverifyefficientanswerverifier}. We adopt $pass@1$ accuracy for all benchmarks except AIME2024, where we report $avg@32$ accuracy to account for its limited size (30 problems) and high difficulty.

We compare GVPO with a wide range of baselines, including GRPO \cite{shao2024deepseekmath}, Dr.GRPO \cite{liu2025understanding}, Remax \cite{li2023remax}, Reinforce++ \cite{hu2025reinforce++}, GSPO \cite{zheng2025groupsequencepolicyoptimization}, GMPO \cite{zhao2025geometricmeanpolicyoptimization}, GPG \cite{chu2025gpg} and Cispo \cite{minimax2025minimaxm1scalingtesttimecompute}. To ensure a fair comparison, we maintain identical experimental settings while only modifying the algorithmic component. All experiments use a same sampling budget of generating $k=5$ responses per prompt.

For each step, we sample $1024$ prompts from the training set and set the mini-batch size in each step to $256$. We repeat the whole training set for 10 epochs and set the warm-up ratio to $5\%$. We grid-search the learning rate in $\{5e-7,\ 1e-6, 5e-6, 1e-5\}$ and find $5e-6$ to be the best setting. For GVPO, we employ $\beta=0.1$. We conduct the main experiment using Qwen2.5-Math-7B on top of Deepseek-R1-like chat template as in \cite{hu2025openreasonerzeroopensourceapproach}. In the ablation experiments, we use Qwen2.5-Math-1.5B on top of Qwen chat template for faster training and GPU memory utilization.

We adopt the verl framework \cite{sheng2024hybridflow} for training. The code implementation of GVPO has been integrated into the verl official GitHub repository. We also show the core implementation of GVPO in Appendix~\ref{app:code}. We conduct our experiments using a server with eight 80GB H800 GPU cards. For Qwen2.5-Math-7B experiments with $k=5$, it takes approximately 12 hours per experiment. For Qwen2.5-Math-1.5B experiments with $k=8$, it takes approximately 8 hours per experiment.

\subsection{Experimental Results}
\textbf{Main Result.} Table~\ref{tab:results} summarizes the main experimental results, comparing different post-training algorithms on Qwen2.5-Math-7B across a range of mathematical reasoning benchmarks. GVPO demonstrates strong performance across the evaluated datasets, achieving competitive results overall.

Beyond its empirical performance, Table~\ref{tab:results} also highlights the key advantages of GVPO. In particular, GVPO guarantees the unique optimal solution and enables off-policy training without requiring importance sampling. To the best of our knowledge, other baselines do not provide these desirable properties. 

The strong empirical performance of GVPO is consistent with our theoretical analysis. Together, the theoretical guarantees and experimental results establish GVPO as a reliable and versatile paradigm for LLM post-training.

\textbf{Ablation on $\beta$.}  Figure~\ref{fig:beta} analyzes the sensitivity of GVPO to variations in $\beta$, with each line corresponding to a different evaluated dataset. The results demonstrate little fluctuation in performance across $\beta$, suggesting GVPO exhibits robustness to this hyperparameter. This stability may reduce the need for exhaustive tuning and enhance its practical utility.

\textbf{Ablation on $k$.} Figure~\ref{fig:k} examines how GVPO scales with the number of samplings $k$, evaluated on Qwen2.5-Math-1.5B. Top and bottom panels show results for MATH500 and AIME2024 respectively. GVPO consistently outperforms GRPO across all $k$ and demonstrates superior scalability. Notably, GVPO matches the AIME2024 performance of a 7B model on the 1.5B architecture through increased $k$, demonstrating its potential to achieve better performance through increased sampling, without requiring larger model sizes.

\textbf{Ablation on $\pi_s$.} Figure~\ref{fig:pi_s} investigates GVPO's versatility on sampling distributions, evaluated on Qwen2.5-Math-1.5B and MATH500. We propose a heuristic $\pi_s$ that mixes responses from $\pi_{\theta_{\text{old}}}$ with historical responses. Specifically, each line corresponds to a different portfolio, where, for example, a ratio of $2{:}6$ indicates that two historical responses are combined with six responses from $\pi_{\theta_{\text{old}}}$. The results demonstrate GVPO’s robust performance across proportions, highlighting: (1) incorporating historical responses into $\pi_s$ can reduce the training's sampling cost, and (2) the flexibility of $\pi_s$ suggests that GVPO provides a promising bridge between modern LLM training and classical RL research on experience replay.

\begin{table*}[htbp]
\centering
\caption{Ablation of GVPO Regularization Terms}
\label{tab:reg}
\begin{tabular}{lccccc}
\toprule
Algorithm & AIME24 & AMC & MATH500 & Minerva & Olympiadbench \\
\midrule
GVPO     & 20.72 & 62.65 & 83.80 & 45.95 & 46.96  \\
GVPO - Var     & 0.00 & 0.00 & 0.00 & 0.00 & 0.00  \\
GVPO - Cov     & 0.00 & 0.00 & 0.00 & 0.00 & 0.00  \\
GVPO - Var - Cov     & 7.19 & 39.76 & 73.00 & 34.93 & 35.70  \\
\midrule
GVPO - Var + Entropy     & 0.00 & 0.00 & 0.00 & 0.00 & 0.00  \\
GVPO - Var + Entropy (LR=1e-6)& 3.02 & 39.76 & 33.20 & 23.16 & 8.89  \\
\bottomrule
\end{tabular}
\end{table*}

\begin{table*}[htbp]
\centering
\caption{Algorithm Performance Comparison across Different Random Seeds}
\label{tab:seed}
\begin{tabular}{lccccc}
\toprule
Algorithm & AIME2024 & AMC & MATH500 & Minerva & Olympiadbench \\
\midrule
GRPO     & 13.59$\pm$1.11 & 44.34$\pm$2.19 & 76.30$\pm$1.35 & 35.40$\pm$1.04 & 38.65$\pm$0.74  \\
\textbf{GVPO}     & \textbf{15.56$\pm$1.05} & \textbf{45.78$\pm$2.05} & \textbf{77.36$\pm$0.98} & \textbf{36.03$\pm$1.15} & \textbf{39.58$\pm$1.08} \\
\bottomrule
\end{tabular}
\vspace{-3mm}
\end{table*}

\begin{table*}[b]
\centering
\caption{Algorithm Performance Comparison on Reasoning Gym across Multiple Domains}
\label{tab:gym}
\begin{tabular}{lccccc}
\toprule
Algorithm & Algorithmic Thinking & Constraint Satisfaction & Logical Reasoning & Mathematical Domains & Pattern Recognition \\
\midrule
Qwen2.5-1.5B-Instruct & 7.00 & 1.24 & 24.17 & 22.73 & 13.90  \\
\quad+GRPO     & 27.95 & 11.90 & 61.83 & 37.52 & 25.65  \\
\quad+\textbf{GVPO}     & \textbf{38.38} & \textbf{17.83} & \textbf{70.72} & \textbf{63.90} & \textbf{30.07} \\
\bottomrule
\end{tabular}
\end{table*}

\textbf{Ablation on Regularization Terms.} 
Our regularization perspective decomposes GVPO into implicit regularization terms. Table~\ref{tab:reg} presents the ablation results for these regularization terms. We first remove $Var(\log\pi_\theta)$ and $Cov(\log\pi_\theta, \log\pi_{\theta^\prime})$ individually. In both cases, the model fails to converge and generates incoherent outputs, indicating that each term plays a crucial role in stabilizing training. When both regularization terms are removed simultaneously—reducing the objective to $R - \overline{R}$—the model initially converges but diverges after approximately 10\% of the training, further confirming that these regularization components are essential to GVPO’s stability.

\textbf{Study on $Var(\log\pi_\theta)$.} We further investigate the role of $Var(\log\pi_\theta)$ by replacing it with an entropy regularization term. As shown in Table~\ref{tab:reg}, the model again fails to converge and produces incoherent outputs. Lowering the learning rate to $1\text{e}{-6}$ improves stability marginally, yet substituting it with entropy regularization still lead to divergence or suboptimal performance. These findings suggest that although entropy regularization serves a similar stabilizing purpose, it cannot fully substitute $Var(\log\pi_\theta)$. The degradation likely arises from entropy regularization being either too weak—insufficient to suppress extreme updates—or too strong—impeding optimization. This underscores a key limitation of entropy regularization: its sensitivity to coefficient tuning. In contrast, the coefficient for $Var(\log\pi_\theta)$ in GVPO is derived analytically via the optimal solution theorem, eliminating the need for manual tuning and enhancing overall robustness.

\textbf{Robustness Check with Different Random Seeds.} We conducted additional experiments using 10 random seeds for both methods on Qwen2.5-Math-1.5B. The results in Table~\ref{tab:seed} show that GVPO consistently outperforms GRPO in overall performance while exhibiting comparable standard deviations, indicating stable results across different runs.

\begin{table}[htbp]
\vspace{-1mm}
\centering
\caption{Performance Comparison on Llama-3.1-8B-Instruct}
\label{tab:base}
\begin{tabular}{lccccc}
\toprule
Algorithm & AIME24 & AMC & MATH500 & Minerva & Olympiad \\
\midrule
Llama & 5.00 & 19.27 & 50.40 & 22.42 & 17.04  \\
\quad+GRPO     & 8.54 & 19.27 & 52.60 & 25.37 & 16.15  \\
\quad+\textbf{GVPO}     & \textbf{11.56} & \textbf{20.48} & \textbf{56.60} & \textbf{29.41} & \textbf{20.59} \\
\bottomrule
\end{tabular}
\vspace{-1mm}
\end{table}

\textbf{Robustness Check with a Different Foundation Model.} To further assess generalization, we repeated the experiments using Llama-3.1-8B-Instruct as the foundation model. As shown in Table~\ref{tab:base}, GVPO again outperforms GRPO, reaffirming the robustness and effectiveness of our approach.

\textbf{Robustness Check with Multiple Domains.} We further validate GVPO on Reasoning Gym \cite{stojanovski2025reasoninggymreasoningenvironments}, a diverse benchmark spanning over 100 tasks across multiple domains. As shown in Table~\ref{tab:gym} and Appendix~\ref{app:gym}, GVPO consistently delivers clear improvements, providing further evidence of its general applicability. We additionally evaluate GVPO on an alignment-oriented summarization task, with results reported in Appendix~\ref{app:sum}.

\newpage
\begin{table*}[htbp]
\centering
\caption{Performance comparison of on-policy distillation algorithms across mathematical reasoning benchmarks.}
\label{tab:opd_results}
\begin{tabular}{lcccccc}
\toprule
Algorithm & AIME2024 & AMC & MATH500 & Minerva & Olympiadbench & Average \\
\midrule
\multicolumn{7}{c}{\textit{Teacher: DeepSeek-R1-Distill-Llama-8B}} \\
SFT  & 1.66 & 15.66 & 49.40 & 15.44 & 15.70 & 19.57 \\
Seqkd  & 1.45 & 16.86 & 45.40 & 15.44 & 16.14 & 19.06 \\
PG  & 7.29 & 27.71 & 49.20 & 16.17 & 23.40 & 24.75 \\
\textbf{GVPO} & \textbf{9.06} & \textbf{39.75} & \textbf{68.40} & \textbf{29.41} & \textbf{33.33} & \textbf{35.99} \\
\midrule
\multicolumn{7}{c}{\textit{Teacher: DeepSeek-R1-Distill-Qwen-7B}} \\
SFT  & 1.56 & 14.45 & 48.60 & 18.75 & 18.66 & 20.40 \\
Seqkd  & 1.14 & 18.07 & 48.40 & 19.48 & 16.74 & 20.77 \\
PG  & 7.60 & 33.73 & 52.20 & 20.58 & 24.88 & 27.80 \\
SupervisedKD  & 1.25 & 10.84 & 42.60 & 8.82 & 14.81 & 15.66 \\
GKD  & 5.20 & 22.89 & 59.60 & 18.38 & 20.29 & 25.27 \\
Minillm  & 7.18 & 24.09 & 55.00 & 20.22 & 24.59 & 26.22 \\
\textbf{GVPO} & \textbf{8.95} & \textbf{34.93} & \textbf{70.20} & \textbf{30.51} & \textbf{32.14} & \textbf{35.35} \\
\bottomrule
\end{tabular}
\end{table*}

\newpage
\section{On-Policy Distillation Experiments}
\subsection{Experimental Setup}
In accord with the general post-training experiments, we conduct OPD experiments on math reasoning, utilizing Competition Math dataset for training and AIME24, AMC, Math500, Minerva, and OlympiadBench for evaluation.

We adopt Qwen2.5-Math-1.5B as the student model. For teacher models, we adopt DeepSeek-R1-Distill-Llama-8B and DeepSeek-R1-Distill-Qwen-7B, corresponding the different-vocabulary scenario and same-vocabulary scenario respectively.

We use verl as our training framework. We adopt the on-policy sampling setting, where we sample $256$ prompts and set the mini-batch size in each step to $256$ for each step. The training lasts for 2 epoch. For each method, we grid-search and select the optimal learning rate in $\{1e-6, 3e-6, 1e-5\}$. All experiments generate $k = 4$ responses per prompt.

As GVPO supports a broad family of the extended OPD objectives, we adopt the weighting scheme $f(x,y)=\frac{1}{|y|^\alpha}$ illustrated in Section~\ref{sec:opd_example} to facilitate concrete experiments. For DeepSeek-R1-Distill-Llama-8B, we adopt $\alpha=1.0$; for DeepSeek-R1-Distill-Qwen-7B, we adopt $\alpha=0.75$.

\subsection{Experimental Results}

\textbf{Main Result.} In the different-vocabulary scenario, we compare GVPO with SFT, SeqKD \cite{kim2016sequence}, and policy gradient. In the same-vocabulary scenario, we additionally compare it with SupervisedKD \cite{hinton2015distilling}, GKD \cite{agarwal2024policy}, and MiniLLM \cite{gu2024minillm}.

Table~\ref{tab:opd_results} presents the experimental results. GVPO consistently achieves the best performance across datasets in both scenarios, with improvements up to $10\%$ and even more.

The advantage of GVPO in the different-vocabulary scenario demonstrates its ability to perform distillation between models from different model families or series whose tokenizers and vocabularies are incompatible. Its advantage in the same-vocabulary scenario further demonstrates that GVPO can achieve superior performance even when compared with methods that assume a shared vocabulary.

We emphasize that GVPO should not be understood simply as adopting $f(x,y)=\frac{1}{|y|^\alpha}$, despite the strong empirical performance of this particular weighting scheme. Rather, this weighting scheme serves as a proof of concept for the extended OPD objectives. It demonstrates the potential of designing heuristic objectives beyond the standard reverse KL divergence, while GVPO provides a practical framework for optimizing such objectives with theoretical guarantees.

\textbf{Ablation on $\alpha$.} We further investigate the weighting scheme studied in this paper by analyzing the effect of the length normalization hyperparameter $\alpha$. Figure~\ref{fig:alpha_accuracy} presents the performance of GVPO across different values of $\alpha$, with each line corresponding to a different evaluation dataset. The left and right panels show results using DeepSeek-R1-Distill-Llama-8B and DeepSeek-R1-Distill-Qwen-7B as teachers, respectively. Overall, the results indicate that applying little or no length control ($\alpha = 0$ or $0.25$) generally leads to inferior performance compared with stronger length control ($\alpha > 0.25$). However, performance does not necessarily improve monotonically with increasing $\alpha$: in the Qwen-based experiments, for example, the best performance is achieved at $\alpha=0.75$. These results suggest that a moderately large $\alpha$ provides a reasonable starting point, while further tuning of this hyperparameter may yield additional performance gains.

\begin{figure}[h]
    \centering

    \begin{minipage}[b]{0.48\columnwidth}
        \centering
        \begin{tikzpicture}
        \begin{axis}[
            xlabel=$\alpha$,
            ylabel=Accuracy,
            xtick={0,0.25,0.5,0.75,1.0},
            xmin=0,
            xmax=1.0,
            ymin=0,
            ymax=0.8,
            grid=both,
            ticklabel style={font=\tiny},
            label style={font=\small},
            width=5.2cm,
            height=4.5cm,
        ]

        % AIME2024
        \addplot+[smooth, black, mark=*] coordinates {
            (0.0, 0.059375)
            (0.25, 0.051042)
            (0.5, 0.076042)
            (0.75, 0.080208)
            (1.0, 0.090625)
        };

        % AMC
        \addplot+[smooth, black, mark=square*] coordinates {
            (0.0, 0.289157)
            (0.25, 0.277108)
            (0.5, 0.373494)
            (0.75, 0.373494)
            (1.0, 0.397590)
        };

        % MATH500
        \addplot+[smooth, black, mark=triangle*] coordinates {
            (0.0, 0.504000)
            (0.25, 0.540000)
            (0.5, 0.688000)
            (0.75, 0.676000)
            (1.0, 0.684000)
        };

        % Minerva
        \addplot+[smooth, black, mark=diamond*] coordinates {
            (0.0, 0.194853)
            (0.25, 0.205882)
            (0.5, 0.283088)
            (0.75, 0.283088)
            (1.0, 0.294118)
        };

        % Olympiadbench
        \addplot+[smooth, black, mark=pentagon*] coordinates {
            (0.0, 0.220741)
            (0.25, 0.228148)
            (0.5, 0.311111)
            (0.75, 0.302222)
            (1.0, 0.333333)
        };

        \end{axis}
        \end{tikzpicture}
    \end{minipage}
    \hfill
    \begin{minipage}[b]{0.48\columnwidth}
        \centering
        \begin{tikzpicture}
        \begin{axis}[
            xlabel=$\alpha$,
            ylabel={},
            yticklabels={},
            xtick={0,0.25,0.5,0.75,1.0},
            xmin=0,
            xmax=1.0,
            ymin=0,
            ymax=0.8,
            grid=both,
            ticklabel style={font=\tiny},
            label style={font=\small},
            width=5.2cm,
            height=4.5cm,
        ]

        % AIME2024
        \addplot+[smooth, black, mark=*] coordinates {
            (0.0, 0.078125)
            (0.25, 0.061458)
            (0.5, 0.064583)
            (0.75, 0.089583)
            (1.0, 0.045833)
        };

        % AMC
        \addplot+[smooth, black, mark=square*] coordinates {
            (0.0, 0.289157)
            (0.25, 0.277108)
            (0.5, 0.325301)
            (0.75, 0.349398)
            (1.0, 0.325301)
        };

        % MATH500
        \addplot+[smooth, black, mark=triangle*] coordinates {
            (0.0, 0.526000)
            (0.25, 0.520000)
            (0.5, 0.658000)
            (0.75, 0.702000)
            (1.0, 0.436000)
        };

        % Minerva
        \addplot+[smooth, black, mark=diamond*] coordinates {
            (0.0, 0.187500)
            (0.25, 0.158088)
            (0.5, 0.316176)
            (0.75, 0.305147)
            (1.0, 0.158088)
        };

        % Olympiadbench
        \addplot+[smooth, black, mark=pentagon*] coordinates {
            (0.0, 0.232593)
            (0.25, 0.225185)
            (0.5, 0.284444)
            (0.75, 0.321481)
            (1.0, 0.208889)
        };

        \end{axis}
        \end{tikzpicture}
    \end{minipage}

    \caption{Ablation on $\alpha$. Each line represents a dataset. The left and right figures show results for Llama and Qwen, respectively.}
    \label{fig:alpha_accuracy}
    \vspace{-2mm}
\end{figure}
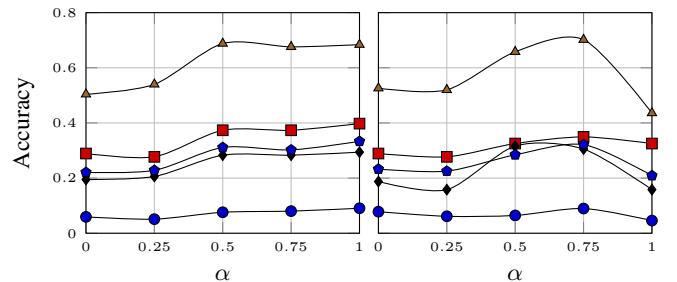

\section{Conclusion}

In this paper, we present Group Variance Policy Optimization (GVPO). GVPO guarantees a unique optimal solution, exactly the KL-constrained reward maximization objective. Moreover, it supports flexible sampling distributions that avoids the limitations of on-policy training and importance sampling. Through systematic comparisons with other prominent methods both theoretically and empirically, we establish GVPO as a new paradigm for reliable and versatile LLM post-training and on-policy distillation.

\clearpage

\bibliographystyle{IEEEtran}
\bibliography{main}

@misc{hu2025openreasonerzeroopensourceapproach,
      title={Open-Reasoner-Zero: An Open Source Approach to Scaling Up Reinforcement Learning on the Base Model}, 
      author={Jingcheng Hu and Yinmin Zhang and Qi Han and Daxin Jiang and Xiangyu Zhang and Heung-Yeung Shum},
      year={2025},
      eprint={2503.24290},
      archivePrefix={arXiv},
      primaryClass={cs.LG},
      url={https://arxiv.org/abs/2503.24290}, 
}

@article{sheng2024hybridflow,
  title   = {HybridFlow: A Flexible and Efficient RLHF Framework},
  author  = {Guangming Sheng and Chi Zhang and Zilingfeng Ye and Xibin Wu and Wang Zhang and Ru Zhang and Yanghua Peng and Haibin Lin and Chuan Wu},
  year    = {2024},
  journal = {arXiv preprint arXiv: 2409.19256}
}

@misc{chen2025xverifyefficientanswerverifier,
      title={xVerify: Efficient Answer Verifier for Reasoning Model Evaluations}, 
      author={Ding Chen and Qingchen Yu and Pengyuan Wang and Wentao Zhang and Bo Tang and Feiyu Xiong and Xinchi Li and Minchuan Yang and Zhiyu Li},
      year={2025},
      eprint={2504.10481},
      archivePrefix={arXiv},
      primaryClass={cs.CL},
      url={https://arxiv.org/abs/2504.10481}, 
}

@InProceedings{pmlr-v37-schulman15,
  title = 	 {Trust Region Policy Optimization},
  author = 	 {Schulman, John and Levine, Sergey and Abbeel, Pieter and Jordan, Michael and Moritz, Philipp},
  booktitle = 	 {Proceedings of the 32nd International Conference on Machine Learning},
  pages = 	 {1889--1897},
  year = 	 {2015},
  editor = 	 {Bach, Francis and Blei, David},
  volume = 	 {37},
  series = 	 {Proceedings of Machine Learning Research},
  address = 	 {Lille, France},
  month = 	 {07--09 Jul},
  publisher =    {PMLR},
  url = 	 {https://proceedings.mlr.press/v37/schulman15.html}
}

@misc{zhao2025surveylargelanguagemodels,
      title={A Survey of Large Language Models}, 
      author={Wayne Xin Zhao and Kun Zhou and Junyi Li and Tianyi Tang and Xiaolei Wang and Yupeng Hou and Yingqian Min and Beichen Zhang and Junjie Zhang and Zican Dong and Yifan Du and Chen Yang and Yushuo Chen and Zhipeng Chen and Jinhao Jiang and Ruiyang Ren and Yifan Li and Xinyu Tang and Zikang Liu and Peiyu Liu and Jian-Yun Nie and Ji-Rong Wen},
      year={2025},
      eprint={2303.18223},
      archivePrefix={arXiv},
      primaryClass={cs.CL},
      url={https://arxiv.org/abs/2303.18223}, 
}

@misc{tie2025surveyposttraininglargelanguage,
      title={A Survey on Post-training of Large Language Models}, 
      author={Guiyao Tie and Zeli Zhao and Dingjie Song and Fuyang Wei and Rong Zhou and Yurou Dai and Wen Yin and Zhejian Yang and Jiangyue Yan and Yao Su and Zhenhan Dai and Yifeng Xie and Yihan Cao and Lichao Sun and Pan Zhou and Lifang He and Hechang Chen and Yu Zhang and Qingsong Wen and Tianming Liu and Neil Zhenqiang Gong and Jiliang Tang and Caiming Xiong and Heng Ji and Philip S. Yu and Jianfeng Gao},
      year={2025},
      eprint={2503.06072},
      archivePrefix={arXiv},
      primaryClass={cs.CL},
      url={https://arxiv.org/abs/2503.06072}, 
}

@misc{zhou2023comprehensivesurveypretrainedfoundation,
      title={A Comprehensive Survey on Pretrained Foundation Models: A History from BERT to ChatGPT}, 
      author={Ce Zhou and Qian Li and Chen Li and Jun Yu and Yixin Liu and Guangjing Wang and Kai Zhang and Cheng Ji and Qiben Yan and Lifang He and Hao Peng and Jianxin Li and Jia Wu and Ziwei Liu and Pengtao Xie and Caiming Xiong and Jian Pei and Philip S. Yu and Lichao Sun},
      year={2023},
      eprint={2302.09419},
      archivePrefix={arXiv},
      primaryClass={cs.AI},
      url={https://arxiv.org/abs/2302.09419}, 
}

@article{ouyang2022training,
  title={Training language models to follow instructions with human feedback},
  author={Ouyang, Long and Wu, Jeffrey and Jiang, Xu and Almeida, Diogo and Wainwright, Carroll and Mishkin, Pamela and Zhang, Chong and Agarwal, Sandhini and Slama, Katarina and Ray, Alex and others},
  journal={Advances in neural information processing systems},
  volume={35},
  pages={27730--27744},
  year={2022}
}

@article{bai2022training,
  title={Training a helpful and harmless assistant with reinforcement learning from human feedback},
  author={Bai, Yuntao and Jones, Andy and Ndousse, Kamal and Askell, Amanda and Chen, Anna and DasSarma, Nova and Drain, Dawn and Fort, Stanislav and Ganguli, Deep and Henighan, Tom and others},
  journal={arXiv preprint arXiv:2204.05862},
  year={2022}
}

@article{shao2024deepseekmath,
  title={Deepseekmath: Pushing the limits of mathematical reasoning in open language models},
  author={Shao, Zhihong and Wang, Peiyi and Zhu, Qihao and Xu, Runxin and Song, Junxiao and Bi, Xiao and Zhang, Haowei and Zhang, Mingchuan and Li, YK and Wu, Y and others},
  journal={arXiv preprint arXiv:2402.03300},
  year={2024}
}

@article{schulman2017proximal,
  title={Proximal policy optimization algorithms},
  author={Schulman, John and Wolski, Filip and Dhariwal, Prafulla and Radford, Alec and Klimov, Oleg},
  journal={arXiv preprint arXiv:1707.06347},
  year={2017}
}

@article{guo2025deepseek,
  title={Deepseek-r1: Incentivizing reasoning capability in llms via reinforcement learning},
  author={Guo, Daya and Yang, Dejian and Zhang, Haowei and Song, Junxiao and Zhang, Ruoyu and Xu, Runxin and Zhu, Qihao and Ma, Shirong and Wang, Peiyi and Bi, Xiao and others},
  journal={arXiv preprint arXiv:2501.12948},
  year={2025}
}

@article{yu2025dapo,
  title={Dapo: An open-source llm reinforcement learning system at scale},
  author={Yu, Qiying and Zhang, Zheng and Zhu, Ruofei and Yuan, Yufeng and Zuo, Xiaochen and Yue, Yu and Fan, Tiantian and Liu, Gaohong and Liu, Lingjun and Liu, Xin and others},
  journal={arXiv preprint arXiv:2503.14476},
  year={2025}
}

@article{liu2025understanding,
  title={Understanding r1-zero-like training: A critical perspective},
  author={Liu, Zichen and Chen, Changyu and Li, Wenjun and Qi, Penghui and Pang, Tianyu and Du, Chao and Lee, Wee Sun and Lin, Min},
  journal={arXiv preprint arXiv:2503.20783},
  year={2025}
}

@article{gao2024towards,
  title={Towards a unified view of preference learning for large language models: A survey},
  author={Gao, Bofei and Song, Feifan and Miao, Yibo and Cai, Zefan and Yang, Zhe and Chen, Liang and Hu, Helan and Xu, Runxin and Dong, Qingxiu and Zheng, Ce and others},
  journal={arXiv preprint arXiv:2409.02795},
  year={2024}
}

@article{touvron2023llama,
  title={Llama 2: Open foundation and fine-tuned chat models},
  author={Touvron, Hugo and Martin, Louis and Stone, Kevin and Albert, Peter and Almahairi, Amjad and Babaei, Yasmine and Bashlykov, Nikolay and Batra, Soumya and Bhargava, Prajjwal and Bhosale, Shruti and others},
  journal={arXiv preprint arXiv:2307.09288},
  year={2023}
}

@article{rafailov2023direct,
  title={Direct preference optimization: Your language model is secretly a reward model},
  author={Rafailov, Rafael and Sharma, Archit and Mitchell, Eric and Manning, Christopher D and Ermon, Stefano and Finn, Chelsea},
  journal={Advances in Neural Information Processing Systems},
  volume={36},
  pages={53728--53741},
  year={2023}
}

@article{bradley1952rank,
  title={Rank analysis of incomplete block designs: I. The method of paired comparisons},
  author={Bradley, Ralph Allan and Terry, Milton E},
  journal={Biometrika},
  volume={39},
  number={3/4},
  pages={324--345},
  year={1952},
  publisher={JSTOR}
}

@inproceedings{jaques2017sequence,
  title={Sequence tutor: Conservative fine-tuning of sequence generation models with kl-control},
  author={Jaques, Natasha and Gu, Shixiang and Bahdanau, Dzmitry and Hern{\'a}ndez-Lobato, Jos{\'e} Miguel and Turner, Richard E and Eck, Douglas},
  booktitle={International Conference on Machine Learning},
  pages={1645--1654},
  year={2017},
  organization={PMLR}
}

@InProceedings{hong2024energybasedpreferencemodeloffers,
  title = 	 {Energy-Based Preference Model Offers Better Offline Alignment than the Bradley-Terry Preference Model},
  author =       {Hong, Yuzhong and Zhang, Hanshan and Bao, Junwei and Jiang, Hongfei and Song, Yang},
  booktitle = 	 {Proceedings of the 42nd International Conference on Machine Learning},
  pages = 	 {23787--23804},
  year = 	 {2025},
  volume = 	 {267},
  series = 	 {Proceedings of Machine Learning Research},
  month = 	 {13--19 Jul},
  publisher =    {PMLR},
  url = 	 {https://proceedings.mlr.press/v267/hong25i.html},
}

@article{tang2024generalized,
  title={Generalized preference optimization: A unified approach to offline alignment},
  author={Tang, Yunhao and Guo, Zhaohan Daniel and Zheng, Zeyu and Calandriello, Daniele and Munos, R{\'e}mi and Rowland, Mark and Richemond, Pierre Harvey and Valko, Michal and Pires, Bernardo {\'A}vila and Piot, Bilal},
  journal={arXiv preprint arXiv:2402.05749},
  year={2024}
}

@inproceedings{bong2022generalized,
  title={Generalized results for the existence and consistency of the MLE in the Bradley-Terry-Luce model},
  author={Bong, Heejong and Rinaldo, Alessandro},
  booktitle={International Conference on Machine Learning},
  pages={2160--2177},
  year={2022},
  organization={PMLR}
}

@article{williams1992simple,
  title={Simple statistical gradient-following algorithms for connectionist reinforcement learning},
  author={Williams, Ronald J},
  journal={Machine learning},
  volume={8},
  pages={229--256},
  year={1992},
  publisher={Springer}
}

@inproceedings{ahmed2019understanding,
  title={Understanding the impact of entropy on policy optimization},
  author={Ahmed, Zafarali and Le Roux, Nicolas and Norouzi, Mohammad and Schuurmans, Dale},
  booktitle={International conference on machine learning},
  pages={151--160},
  year={2019},
  organization={PMLR}
}

@misc{wu2022asymptoticcomparisonidentifyingconstraints,
      title={Asymptotic comparison of identifying constraints for Bradley-Terry models}, 
      author={Weichen Wu and Brian W. Junker and Nynke M. D. Niezink},
      year={2022},
      eprint={2205.04341},
      archivePrefix={arXiv},
      primaryClass={math.ST},
      url={https://arxiv.org/abs/2205.04341}, 
}

@article{singh2000convergence,
  title={Convergence results for single-step on-policy reinforcement-learning algorithms},
  author={Singh, Satinder and Jaakkola, Tommi and Littman, Michael L and Szepesv{\'a}ri, Csaba},
  journal={Machine learning},
  volume={38},
  pages={287--308},
  year={2000},
  publisher={Springer}
}

@article{hendrycksmath2021,
    title={Measuring Mathematical Problem Solving With the MATH Dataset},
    author={Dan Hendrycks
    and Collin Burns
    and Saurav Kadavath
    and Akul Arora
    and Steven Basart
    and Eric Tang
    and Dawn Song
    and Jacob Steinhardt},
    journal={arXiv preprint arXiv:2103.03874},
    year={2021}
}

@article{li2024numinamath,
  title={Numinamath: The largest public dataset in ai4maths with 860k pairs of competition math problems and solutions},
  author={Li, Jia and Beeching, Edward and Tunstall, Lewis and Lipkin, Ben and Soletskyi, Roman and Huang, Shengyi and Rasul, Kashif and Yu, Longhui and Jiang, Albert Q and Shen, Ziju and others},
  journal={Hugging Face repository},
  volume={13},
  pages={9},
  year={2024}
}

@article{lewkowycz2022solving,
  title={Solving quantitative reasoning problems with language models},
  author={Lewkowycz, Aitor and Andreassen, Anders and Dohan, David and Dyer, Ethan and Michalewski, Henryk and Ramasesh, Vinay and Slone, Ambrose and Anil, Cem and Schlag, Imanol and Gutman-Solo, Theo and others},
  journal={Advances in Neural Information Processing Systems},
  volume={35},
  pages={3843--3857},
  year={2022}
}

@article{he2024olympiadbench,
  title={Olympiadbench: A challenging benchmark for promoting agi with olympiad-level bilingual multimodal scientific problems},
  author={He, Chaoqun and Luo, Renjie and Bai, Yuzhuo and Hu, Shengding and Thai, Zhen Leng and Shen, Junhao and Hu, Jinyi and Han, Xu and Huang, Yujie and Zhang, Yuxiang and others},
  journal={arXiv preprint arXiv:2402.14008},
  year={2024}
}

@article{hu2025reinforce++,
  title={Reinforce++: A simple and efficient approach for aligning large language models},
  author={Hu, Jian},
  journal={arXiv preprint arXiv:2501.03262},
  year={2025}
}

@article{zhang2025rspo,
  title={RSPO: Risk-Seeking Policy Optimization for Pass@ k and Max@ k Metrics in Large Language Models},
  author={Zhang, Kaichen and Gao, Shenghao and Hong, Yuzhong and Sun, Haipeng and Bao, Junwei and Jiang, Hongfei and Song, Yang and Dingqian, Hong and Xiong, Hui},
  journal={arXiv preprint arXiv:2508.01174},
  year={2025}
}

@inproceedings{fedus2020revisiting,
  title={Revisiting fundamentals of experience replay},
  author={Fedus, William and Ramachandran, Prajit and Agarwal, Rishabh and Bengio, Yoshua and Larochelle, Hugo and Rowland, Mark and Dabney, Will},
  booktitle={International conference on machine learning},
  pages={3061--3071},
  year={2020},
  organization={PMLR}
}

@article{tokdar2010importance,
  title={Importance sampling: a review},
  author={Tokdar, Surya T and Kass, Robert E},
  journal={Wiley Interdisciplinary Reviews: Computational Statistics},
  volume={2},
  number={1},
  pages={54--60},
  year={2010},
  publisher={Wiley Online Library}
}

@article{qian2025beyond,
  title={Beyond the Next Token: Towards Prompt-Robust Zero-Shot Classification via Efficient Multi-Token Prediction},
  author={Qian, Junlang and Zhu, Zixiao and Zhou, Hanzhang and Feng, Zijian and Zhai, Zepeng and Mao, Kezhi},
  journal={arXiv preprint arXiv:2504.03159},
  year={2025}
}

@book{miao2024optimizing,
  title={Optimizing the Unknown: Reinforcement Learning and Energy-Based Model for Black Box Bayesian Optimization},
  author={Miao, Ruiyao},
  year={2024},
  publisher={University of California, Los Angeles}
}

@article{li2023remax,
  title={Remax: A simple, effective, and efficient reinforcement learning method for aligning large language models},
  author={Li, Ziniu and Xu, Tian and Zhang, Yushun and Lin, Zhihang and Yu, Yang and Sun, Ruoyu and Luo, Zhi-Quan},
  journal={arXiv preprint arXiv:2310.10505},
  year={2023}
}

@article{skalse2022defining,
  title={Defining and characterizing reward gaming},
  author={Skalse, Joar and Howe, Nikolaus and Krasheninnikov, Dmitrii and Krueger, David},
  journal={Advances in neural information processing systems},
  volume={35},
  pages={9460--9471},
  year={2022}
}

@article{zhang2026gvpo,
  title={Gvpo: Group variance policy optimization for large language model post-training},
  author={Zhang, Kaichen and Hong, Yuzhong and Bao, Junwei and Jiang, Hongfei and Song, Yang and Dingqian, Hong and Xiong, Hui},
  journal={Advances in Neural Information Processing Systems},
  volume={38},
  pages={165798--165820},
  year={2026}
}

@misc{song2026surveyonpolicydistillationlarge,
      title={A Survey of On-Policy Distillation for Large Language Models}, 
      author={Mingyang Song and Mao Zheng},
      year={2026},
      eprint={2604.00626},
      archivePrefix={arXiv},
      primaryClass={cs.LG},
      url={https://arxiv.org/abs/2604.00626}, 
}

@article{gou2021knowledge,
  title={Knowledge distillation: A survey},
  author={Gou, Jianping and Yu, Baosheng and Maybank, Stephen J and Tao, Dacheng},
  journal={International journal of computer vision},
  volume={129},
  number={6},
  pages={1789--1819},
  year={2021},
  publisher={Springer}
}

@inproceedings{gu2024minillm,
  title={Minillm: Knowledge distillation of large language models},
  author={Gu, Yuxian and Dong, Li and Wei, Furu and Huang, Minlie},
  booktitle={International Conference on Learning Representations},
  volume={2024},
  pages={32694--32717},
  year={2024}
}

@inproceedings{agarwal2024policy,
  title={On-policy distillation of language models: Learning from self-generated mistakes},
  author={Agarwal, Rishabh and Vieillard, Nino and Zhou, Yongchao and Stanczyk, Piotr and Ramos Garea, Sabela and Geist, Matthieu and Bachem, Olivier},
  booktitle={International Conference on Learning Representations},
  volume={2024},
  pages={21246--21263},
  year={2024}
}

@inproceedings{kim2016sequence,
  title={Sequence-level knowledge distillation},
  author={Kim, Yoon and Rush, Alexander M},
  booktitle={Proceedings of the 2016 conference on empirical methods in natural language processing},
  pages={1317--1327},
  year={2016}
}

@article{hinton2015distilling,
  title={Distilling the knowledge in a neural network},
  author={Hinton, Geoffrey and Vinyals, Oriol and Dean, Jeff},
  journal={arXiv preprint arXiv:1503.02531},
  year={2015}
}

@misc{zheng2025groupsequencepolicyoptimization,
      title={Group Sequence Policy Optimization}, 
      author={Chujie Zheng and Shixuan Liu and Mingze Li and Xiong-Hui Chen and Bowen Yu and Chang Gao and Kai Dang and Yuqiong Liu and Rui Men and An Yang and Jingren Zhou and Junyang Lin},
      year={2025},
      eprint={2507.18071},
      archivePrefix={arXiv},
      primaryClass={cs.LG},
      url={https://arxiv.org/abs/2507.18071}, 
}

@misc{zhao2025geometricmeanpolicyoptimization,
      title={Geometric-Mean Policy Optimization}, 
      author={Yuzhong Zhao and Yue Liu and Junpeng Liu and Jingye Chen and Xun Wu and Yaru Hao and Tengchao Lv and Shaohan Huang and Lei Cui and Qixiang Ye and Fang Wan and Furu Wei},
      year={2025},
      eprint={2507.20673},
      archivePrefix={arXiv},
      primaryClass={cs.CL},
      url={https://arxiv.org/abs/2507.20673}, 
}

@article{chu2025gpg,
  title={Gpg: A simple and strong reinforcement learning baseline for model reasoning},
  author={Chu, Xiangxiang and Huang, Hailang and Zhang, Xiao and Wei, Fei and Wang, Yong},
  journal={ICLR},
  year={2026}
}

@misc{minimax2025minimaxm1scalingtesttimecompute,
      title={MiniMax-M1: Scaling Test-Time Compute Efficiently with Lightning Attention}, 
      author={MiniMax},
      year={2025},
      eprint={2506.13585},
      archivePrefix={arXiv},
      primaryClass={cs.CL},
      url={https://arxiv.org/abs/2506.13585}, 
}

@misc{stojanovski2025reasoninggymreasoningenvironments,
      title={REASONING GYM: Reasoning Environments for Reinforcement Learning with Verifiable Rewards},
      author={Zafir Stojanovski and Oliver Stanley and Joe Sharratt and Richard Jones and Abdulhakeem Adefioye and Jean Kaddour and Andreas Köpf},
      year={2025},
      eprint={2505.24760},
      archivePrefix={arXiv},
      primaryClass={cs.LG},
      url={https://arxiv.org/abs/2505.24760},
}

@article{kumar2026llm,
  title={Llm post-training: A deep dive into reasoning large language models},
  author={Kumar, Komal and Ashraf, Tajamul and Thawakar, Omkar and Anwer, Rao Muhammad and Cholakkal, Hisham and Shah, Mubarak and Yang, Ming-Hsuan and Torr, Phillip HS and Khan, Fahad Shahbaz and Khan, Salman},
  journal={IEEE Transactions on Pattern Analysis and Machine Intelligence},
  year={2026},
  publisher={IEEE}
}

@article{zheng2026lifelong,
  title={Lifelong learning of large language model based agents: A roadmap},
  author={Zheng, Junhao and Shi, Chengming and Cai, Xidi and Li, Qiuke and Zhang, Duzhen and Li, Chenxing and Yu, Dong and Ma, Qianli},
  journal={IEEE Transactions on Pattern Analysis and Machine Intelligence},
  year={2026},
  publisher={IEEE}
}

@article{xu2026parameter,
  title={Parameter-efficient fine-tuning methods for pretrained language models: A critical review and assessment},
  author={Xu, Lingling and Xie, Haoran and Qin, S Joe and Tao, Xiaohui and Wang, Fu Lee},
  journal={IEEE Transactions on Pattern Analysis and Machine Intelligence},
  year={2026},
  publisher={IEEE}
}

@article{besta2025demystifying,
  title={Demystifying chains, trees, and graphs of thoughts},
  author={Besta, Maciej and Memedi, Florim and Zhang, Zhenyu and Gerstenberger, Robert and Piao, Guangyuan and Blach, Nils and Nyczyk, Piotr and Copik, Marcin and Kwa{\'s}niewski, Grzegorz and M{\"u}ller, J{\"u}rgen and others},
  journal={IEEE Transactions on Pattern Analysis and Machine Intelligence},
  volume={47},
  number={12},
  pages={10967--10989},
  year={2025},
  publisher={IEEE}
}

@article{gong2025pushing,
  title={Pushing the limit of post-training quantization},
  author={Gong, Ruihao and Liu, Xianglong and Li, Yuhang and Fan, Yunqiang and Wei, Xiuying and Guo, Jinyang},
  journal={IEEE transactions on pattern analysis and machine intelligence},
  year={2025},
  publisher={IEEE}
}

@article{lightman2023lets,
      title={Let's Verify Step by Step}, 
      author={Lightman, Hunter and Kosaraju, Vineet and Burda, Yura and Edwards, Harri and Baker, Bowen and Lee, Teddy and Leike, Jan and Schulman, John and Sutskever, Ilya and Cobbe, Karl},
      journal={arXiv preprint arXiv:2305.20050},
      year={2023}
}

\clearpage
\onecolumn
\appendix
\subsection{Code Implementation}

\label{app:code}

We present a minimal viable implementation of GVPO based on the framework of verl. In this implementation, each input batch consists of the $k$ responses generated for a given prompt. The advantage computation simply follows that of GRPO without the standard-deviation normalization, yielding $R - \bar{R}$. The implementation then minimizes the mean squared error loss between the implicit reward central distances and their corresponding actual reward central distances.

\begin{lstlisting}[caption={A Simple GVPO Code Implementation}, label={lst:policy_loss}]
def compute_policy_loss(old_log_prob, log_prob, advantages, eos_mask, **kwargs):
    scores = (log_prob * eos_mask).sum(dim=-1)
    scores_old = (old_log_prob * eos_mask).sum(dim=-1)
    advs = (advantages * eos_mask).sum(dim=-1) / eos_mask.sum(dim=-1)
    
    beta = 0.1
    k = scores.size(0)
    
    scores_new = scores.detach()
    loss = -1 * beta * scores * (advs - beta * ((scores_new - scores_new.mean()) - (scores_old - scores_old.mean())))
    
    return loss.mean()
\end{lstlisting}

Moreover, the code implementation of GVPO has been integrated into the verl official GitHub repository.

\section{Proofs}
\label{app:proof}
\subsection{Proof of Theorem~\ref{theorem:gvpo}}
\label{app:proof1}
\textbf{Theorem~\ref{theorem:gvpo}}. \textit{The unique optimal policy that minimizes $\hat{\mathcal{L}}_{\text{GVPO}}(\theta)$, defined as
\[
\hat{\mathcal{L}}_{\text{GVPO}}(\theta)=
\mathbb{E}_{x \sim \mathcal{D}} \mathbb{E}_{y \sim \pi_{s}(\cdot|x)}[(R_\theta(x,y)-
\mathbb{E}_{y \sim \pi_{s}}R_\theta(x,y))-(R(x,y)-
\mathbb{E}_{y \sim \pi_{s}}R(x,y))]^2
\]
, is given by $\pi_\theta (y|x)=\pi^* (y|x)= \frac{1}{Z(x)}\pi_{\theta^\prime}(y|x)e^{R(x,y)/\beta}$ 
for $\pi_s=\pi_{\theta^\prime}$.}

\begin{proof}
We prove the theorem by establishing both necessity and sufficiency.

\textbf{Necessity:} If $\pi_\theta(y|x) = \pi^*(y|x)$, then it is an optimal policy solution.  

The loss function $\hat{\mathcal{L}}_{\text{GVPO}}(\theta)$ is non-negative because it represents the expectation of a squared term:
\[
\hat{\mathcal{L}}_{\text{GVPO}}(\theta) = \mathbb{E}_{x,y} \left[ \big((R_\theta(x,y) - \mathbb{E}_y R_\theta(x,y)) - (R(x,y) - \mathbb{E}_y R(x,y)\big)^2 \right] \geq 0.
\]
When $\pi_\theta(y|x) = \pi^*(y|x)$, we have $R_\theta(x,y) = R(x,y)$. Substituting this into the loss function gives $\hat{\mathcal{L}}_{\text{GVPO}}(\theta) = 0$, confirming that $\pi^*$ achieves the minimum loss.

\textbf{Sufficiency:} If a policy $\pi_\theta$ is optimal, then $\pi_\theta(y|x) = \pi^*(y|x)$.  

Assume for contradiction that there exists an optimal policy $\pi_\theta \neq \pi^*$. Since $\pi_\theta$ is optimal, $\hat{\mathcal{L}}_{\text{GVPO}}(\theta) = 0$. This implies:
\[
(R_\theta(x,y) - \mathbb{E}_y R_\theta(x,y)) = (R(x,y) - \mathbb{E}_y R(x,y)), \quad \forall x, y \text{ s.t. } \pi_s(y|x)>0
\]
Rewriting $R_\theta$ and $R$ in terms of their respective policies:
\[
\beta log \pi_\theta(y|x) - \mathbb{E}_y R_\theta(x,y) = \beta log \pi^*(y|x) - \mathbb{E}_y R(x,y).
\]
Rearranging terms yields:
\[
\pi_\theta(y|x) = \exp\left(\frac{\mathbb{E}_y [R_\theta(x,y) - R(x,y)]}{\beta}\right) \pi^*(y|x).
\]
Since $\sum_{y\in\{y|\pi_{\theta^\prime}(y|x)>0\}} \pi_\theta(y|x) = \sum_{y\in\{y|\pi_{\theta^\prime}(y|x)>0\}} \pi^*(y|x) = 1$, we must have:
\[
\sum_y\pi_\theta(y|x) = \exp\left(\frac{\mathbb{E}_y [R_\theta(x,y) - R(x,y)]}{\beta}\right) \sum_y\pi^*(y|x)
\]
\[
\implies \exp\left(\frac{\mathbb{E}_y [R_\theta(x,y) - R(x,y)]}{\beta}\right) = 1
\]
Thus, $\pi_\theta(y|x) = \pi^*(y|x)$ for all $x, y$, contradicting the assumption $\pi_\theta \neq \pi^*$.  

Since both necessity and sufficiency hold, the optimal policy is uniquely $\pi^*$.
\end{proof}

\subsection{Proof of Theorem~\ref{theorem:gvpo_alg}}
\label{app:proof2}
\textbf{Theorem~\ref{theorem:gvpo_alg}}. \textit{The $n$-step online algorithm, which uses $\hat{\mathcal{L}}_{\text{GVPO}}(\theta_t)$ to iteratively update the initial policy $\pi_{\theta_0}$ by setting $\pi_{\theta'} = \pi_{\theta_{t-1}}$ at each step $t = 1, \dots, n$, maximizes the objective:
\[
\mathbb{E}_{x\sim\mathcal{D}, y\sim\pi_\theta(y|x)}[R(x,y)] - \frac{\beta}{n}\mathbb{D}_{\text{KL}}[\pi_\theta(y|x) \| \pi_{\theta_0}(y|x)].
\]
}

\begin{proof}
By Theorem~\ref{theorem:gvpo}, for each step $t = 1, \ldots, n$, we have:
\[
\beta \log\left(\frac{\pi_{\theta_t}(y|x)}{\pi_{\theta_{t-1}}(y|x)}\right) + \beta \log Z_{t-1}(x) = R(x,y),
\]
where $Z_{t-1}(x) = \sum_{y} \pi_{\theta_{t-1}}(y|x) e^{R(x,y)/\beta}$. Summing these equations for all $t$ from 1 to $n$ yields:
\[
\beta \log\left(\frac{\pi_{\theta_n}(y|x)}{\pi_{\theta_0}(y|x)}\right) + \beta \log \prod_{i=0}^{n-1} Z_i(x) = nR(x,y).
\]
Let $Z_{0:n-1}(x) \triangleq \prod_{i=0}^{n-1} Z_i(x)$. Rearranging terms gives:
\[
\pi_{\theta_n}(y|x) = \frac{1}{Z_{0:n-1}(x)}\pi_{\theta_0}(y|x)e^{nR(x,y)/\beta}.
\]

Next, consider the optimization problem in Equation~\ref{equation:gvpo_alg}:
\begin{equation}
\begin{aligned}
    &\max_{\pi_\theta}\mathbb{E}_{x\sim\mathcal{D},y\sim\pi_\theta(y|x)}[R(x,y)]-\frac{\beta}{n}\mathbb{D}_{KL}[\pi_\theta(y|x)||\pi_{\theta_0}(y|x)]\\
    =&\min_{\pi_\theta}\mathbb{E}_{x\sim\mathcal{D},y\sim\pi_\theta(y|x)}[
    log\frac{\pi_\theta(y|x)}{\pi_{\theta_0}(y|x)}-\frac{n}{\beta}R(x,y)] \\
    =&\min_{\pi_\theta}\mathbb{E}_{x\sim\mathcal{D},y\sim\pi_\theta(y|x)}[
    log\frac{\pi_\theta(y|x)}{\frac{1}{Z_{0:n-1}(x)}\pi_{\theta_0}(y|x)e^{nR(x,y)/\beta}}
    -log Z_{0:n-1}(x)] \\
    =&\min_{\pi_\theta}\mathbb{E}_{x\sim\mathcal{D},y\sim\pi_\theta(y|x)}[
    log\frac{\pi_\theta(y|x)}{\pi_{\theta_n}(y|x)}
    -log Z_{0:n-1}(x)] \\
    =&\min_{\pi_\theta}\mathbb{E}_{x\sim\mathcal{D}}[\mathbb{D}_{KL}(\pi_\theta(y|x)||\pi_{\theta_n}(y|x)) -log Z_{0:n-1}(x)]
\end{aligned}
\end{equation}
The minimum is achieved when the KL divergence is 0, i.e., when $\pi_\theta(y|x) = \pi_{\theta_n}(y|x)$. Hence, the optimal solution to Equation~\ref{equation:gvpo_alg} is $\pi_{\theta_n}(y|x)$.
\end{proof}

\subsection{Evaluations on Summarization} \label{app:sum}
To evaluate GVPO in the context of alignment, we conduct experiments on a summarization task following the experimental setup of DPO. Specifically, we used the Reddit dataset from \url{https://github.com/openai/summarize-from-feedback} and the reward model OpenAssistant/reward-model-deberta-v3-large-v2. First, we fine-tuned Qwen2.5 1.5B using the dataset to obtain an SFT reference model. We then applied both DPO and GVPO for post-training. The training data was sampled from the reference model and scored using the reward model.

For evaluation, we considered:
\begin{itemize}[leftmargin=0.5cm]
    \item The average reward of generated summaries on the test split.
    \item Win rate against human-written demonstration answers, as scored by the reward model.
    \item Preference accuracy using human-labeled comparison data: both DPO and GVPO are trained via
$R_\theta(x, y) = \beta \frac{\log \pi_\theta(y|x)}{\log \pi_{\text{ref}}(y|x)} + \beta \log Z(x)$. Given a preference pair $(y_w, y_l)$ where $y_w$ is preferred, we compute whether
$\frac{\log \pi_\theta(y_w|x)}{\log \pi_{\text{ref}}(y_w|x)} > \frac{\log \pi_\theta(y_l|x)}{\log \pi_{\text{ref}}(y_l|x)}$.
    \item Benchmarking by a set of powerful LLMs.
    \item Human evaluations by crowd-sourcing workers. We hire three workers to label their preference on 100 samples on Prolific platform.
\end{itemize}

\begin{table}[htbp]
\centering
\caption{Algorithm Performance Comparison on Summarization}
\label{tab:sum}
\begin{tabular}{lccccccc}
\toprule
Algorithm & Reward & Win Rate & Accuracy & Gpt-4o & Gemini-2.5-pro & Deepseek-R1 & Human \\
\midrule
SFT & 2.84 & 41.85\% & --- & --- & --- & --- & ---  \\
+DPO & 4.83 & 68.28\% & 60.43\% & 31\% & 40\% & 25\% & 34\%  \\
+GVPO & 5.75 & 79.49\% & 64.93\% & 69\% & 60\% & 75\% & 66\% \\
\bottomrule
\end{tabular}
\end{table}

The results indicate that GVPO outperforms DPO across these metrics, highlighting its effectiveness for alignment beyond the reasoning setting considered in the main text.

\newpage

\begin{table*}[t]
\centering
\caption{Algorithm Performance Comparison on Reasoning Gym across Multiple Domains}
\label{tab:gym_app}
\begin{tabular}{lccccc}
\toprule
Algorithm & Algorithmic Thinking & Constraint Satisfaction & Logical Reasoning & Mathematical Domains & Pattern Recognition \\
\midrule
Qwen2.5-1.5B-Instruct & 7.00 & 1.24 & 24.17 & 22.73 & 13.90  \\
\quad+GRPO     & 27.95 & 11.90 & 61.83 & 37.52 & 25.65  \\
\quad+\textbf{GVPO}     & \textbf{38.38} & \textbf{17.83} & \textbf{70.72} & \textbf{63.90} & \textbf{30.07} \\
\bottomrule
\end{tabular}
\end{table*}

\subsection{Evaluations on Reasoning Gym} \label{app:gym}

We further evaluate GVPO on the Reasoning Gym benchmark, which comprises 104 tasks spanning a diverse range of domains, including algorithmic reasoning, constraint satisfaction, logical reasoning, mathematics, and pattern recognition. Each task provides an algorithmically verifiable environment that supports reinforcement learning–based training and evaluation. Responses receive rewards of 0, 1, or intermediate values according to their degree of correctness. Its broad task coverage, verifiable evaluation protocols, and non-binary reward structure make Reasoning Gym a rigorous testbed for evaluating the generalizability of our proposed method.

For our experimental setup, we construct a training set containing 250 questions per task and a test set containing 10 questions per task. We train each model for five epochs over the training set and sample eight responses ($n=8$) per prompt for all algorithms. We use Qwen2.5-1.5B-Instruct as the base model with a learning rate of $5\times10^{-6}$; all other hyperparameters and experimental configurations are kept consistent with those of the main experiments.

As shown in Table~\ref{tab:gym_app}, GVPO consistently achieves substantial improvements across the evaluated domains. The resulting performance trends closely align with those observed in the main experiments, providing further evidence of the robustness and broad applicability of the proposed framework.

\end{document}